\documentclass{article}

\usepackage{microtype}
\usepackage{graphicx}
\usepackage{subfigure}
\usepackage{booktabs} 

\usepackage{hyperref}

\usepackage[final]{neurips_2025}

\usepackage{thmtools,thm-restate}

\usepackage{mathbbol}
\usepackage{wrapfig}
\usepackage{multirow}

\usepackage{amsmath}
\usepackage{amssymb}
\usepackage{mathtools}
\usepackage{amsthm}

\usepackage{xcolor,colortbl}
\definecolor{darkyellow}{RGB}{178,137,37}
\definecolor{darkblue}{RGB}{57,84,142}

\usepackage[capitalize,noabbrev]{cleveref}

\newif\ifarxiv
\arxivfalse

\newcommand{\rebut}[1]{#1}

\usepackage{amsmath,amsfonts,bm}

\def\eqref#1{equation~\ref{#1}}

\def\1{\bm{1}}

\def\rr{{\textnormal{r}}}

\def\vzero{{\bm{0}}}
\def\vone{{\bm{1}}}

\def\va{{\bm{a}}}
\def\vb{{\bm{b}}}
\def\vc{{\bm{c}}}

\def\ve{{\bm{e}}}
\def\vf{{\bm{f}}}

\def\vo{{\bm{o}}}

\def\vr{{\bm{r}}}
\def\vs{{\bm{s}}}

\def\vu{{\bm{u}}}

\def\vw{{\bm{w}}}

\def\vy{{\bm{y}}}
\def\vz{{\bm{z}}}

\DeclareMathAlphabet{\mathsfit}{\encodingdefault}{\sfdefault}{m}{sl}
\SetMathAlphabet{\mathsfit}{bold}{\encodingdefault}{\sfdefault}{bx}{n}

\newtheorem{definition}{\textbf{Definition}}
\newtheorem{observation}{\textbf{Observation}}
\newtheorem{corollary}{\textbf{Corollary}}
\newtheorem{lemma}{\textbf{Lemma}}
\newtheorem{assumption}{\textbf{Assumption}}
\newtheorem{theorem}{\textbf{Theorem}}
\newtheorem{remark}{\textbf{Remark}}

\def\ee#1{\mathbb{E}\left[#1\right]}
\def\eee#1#2{\mathbb{E}_{#1}\left[#2\right]}

\def\con{\mathbb{C}}

\def\dd{\mathrm{d}}

\def\vrho{\boldsymbol{\rho}}

\def\rr{\mathbb{R}}
\def\zz{\mathbb{Z}}

\usepackage[textsize=tiny]{todonotes}

\title{\underline{Co}mposing \underline{G}lobal \underline{S}olutions to Reasoning Tasks via Algebraic Objects in Neural Nets}

\author{%
  Yuandong Tian \\
  Meta Superintelligence Lab (FAIR) \\
  \texttt{yuandong@meta.com} \\
}

\def\ours{CoGS}

\def\vw{\mathbf{w}}
\def\i{\mathrm{i}}
\def\cc{\mathbb{C}}
\def\vphi{\boldsymbol{\phi}}
\def\cX{\mathcal{X}}
\def\cZ{\mathcal{Z}}

\def\cW{\mathcal{W}}

\def\bone{{(1)}}
\def\btwo{{(2)}}
\def\idx{\mathrm{idx}}
\def\ord{\mathrm{ord}}

\def\vrho{\boldsymbol{\rho}}

\def\c{\mathrm{c}}
\def\n{\mathrm{n}}
\def\g{\mathrm{g}}
\def\conv{*}
\def\con{R}
\def\cC{\mathcal{C}}

\def\bk{{(k)}}
\def\a{\mathrm{a}}
\def\syn{\mathrm{syn}}
\def\one{\mathrm{one}}

\begin{document}

\maketitle

\begin{abstract}
We prove rich algebraic structures of the solution space for 2-layer neural networks with quadratic activation and $L_2$ loss, trained on reasoning tasks in Abelian group (e.g., modular addition). Such a rich structure enables \emph{analytical} construction of global optimal solutions from partial solutions that only satisfy part of the loss, despite its high nonlinearity. We coin the framework as \ours{} (\emph{\underline{Co}mposing \underline{G}lobal \underline{S}olutions}). Specifically, we show that the weight space over different numbers of hidden nodes of the 2-layer network is equipped with a semi-ring algebraic structure, and the loss function to be optimized consists of \emph{sum potentials}, which are ring homomorphisms, allowing partial solutions to be composed into global ones by ring addition and multiplication. Our experiments show that around $95\%$ of the solutions obtained by gradient descent match exactly our theoretical constructions. Although the global solutions constructed only required a small number of hidden nodes, our analysis on gradient dynamics shows that overparameterization asymptotically decouples training dynamics and is beneficial. We further show that training dynamics favors simpler solutions under weight decay, and thus high-order global solutions such as perfect memorization are unfavorable. The code is open sourced\footnote{\url{https://github.com/facebookresearch/luckmatters/tree/yuandong3/ssl/real-dataset}}. 
\end{abstract}

\vspace{-0.15in}
\section{Introduction}
\vspace{-0.15in}
Large Language Models (LLMs) have shown impressive results in various disciplines~\citep{openai2024gpt4technicalreport,claude3,geminiteam2024gemini15unlockingmultimodal,deepseekai2024deepseekv2strongeconomicalefficient,dubey2024llama3herdmodels,jiang2023mistral7b}, while they also make surprising mistakes in basic reasoning tasks~\citep{nezhurina2024alicewonderlandsimpletasks,berglund2023reversal}. Therefore, it remains an open problem whether it can truly do reasoning tasks. On one hand, existing works demonstrate that the models can learn efficient algorithms (e.g., dynamic programming~\citep{ye2024physics} for language structure modeling, etc) and good representations~\citep{jin2024emergentrepresentationsprogramsemantics,wijmans2023emergence}. Some reports emergent behaviors~\citep{wei2022emergent} when scaling up with data and model size. On the other hand, many works also show that LLMs cannot self-correct~\citep{huang2023large}, and cannot generalize very well beyond the training set for simple tasks~\citep{dziri2023faith,yehudai2024can,ouellette2023counting}, let alone complicated planning tasks~\citep{kambhampati2024llmscantplanhelp,xie2024travelplannerbenchmarkrealworldplanning}.   

To understand how the model performs reasoning and further improve its reasoning power, people have been studying simple arithmetic reasoning problems in depth. Modular addition~\citep{nanda2023progress,zhong2024clock}, i.e., predicting $a+b\mod d$ given $a$ and $b$, is a popular one due to its simple and intuitive structure yet surprising behaviors in learning dynamics (e.g., grokking~\citep{power2022grokking}) and learned representations (e.g., Fourier bases~\citep{zhou2024pre}). Most works focus on various metrics to measure the behaviors and extracting interpretable circuits from trained models~\citep{nanda2023progress,varma2023explaining,huang2024unified}. Analytic solutions can be constructed and/or reverse-engineered~\citep{gromov2023grokking,zhong2024clock,nanda2023progress} but it is not clear how to construct a systematic framework to explain and generalize the results. 

In this work, we systematically analyze 2-layer neural networks with quadratic activation and $L_2$ loss on predicting the outcome of group multiplication in Abelian group $G$, which is an extension of modular addition. We find that global solutions can be constructed \emph{algebraically} from small partial solutions that are optimal only for parts of the loss. We achieve this by showing that (1) for the 2-layer network, there exists a \textbf{\emph{semi-ring}} structure over the weights space \emph{across different order} (i.e., number of hidden nodes or network width), with specifically defined addition and multiplication (Sec.~\ref{sec:semi-ring}), and (2) the $L_2$ loss is a function of \textbf{\emph{sum potentials}} (SPs), which are ring homomorphisms (Theorem~\ref{thm:analyticform}) that allow compositions of partial solutions into global ones using ring operations.


As a result, our theoretical framework, named \ours{} (i.e., \emph{\underline{Co}mposing \underline{G}lobal \underline{S}olutions}), successfully constructs two distinct types of Fourier-based global solutions of per-frequency order 4 (or ``$2\times 2$'') and order 6 (or ``$2\times 3$''), and a global solution of order $d^2$ that correspond to perfect memorization. Empirically, we demonstrate that around $95\%$ of the solutions obtained from gradient descent (with weight decay) have the predicted structure and match exactly with our theoretical construction of order-4 and order-6 solutions.  
In addition, we also analyze the training dynamics, and show that the dynamics favors low-order global solutions, since global solutions algebraically connected by ring multiplication can be proven to also be topologically connected. Therefore, high-order solution like perfect memorization is unfavorable in the dynamics. When the network width goes to infinity, the dynamics of sum potentials becomes decoupled, demystifying why overparameterization improves the performance. 

To our best knowledge, we are the first to discover such algebraic structures inside network training, apply it to analyze solutions to reasoning tasks such as modular additions, and show our theoretical constructions occur in actual gradient descent solutions.

\vspace{-0.15in}
\section{Related Works}
\vspace{-0.15in}
\textbf{Algebraic structures for maching learning}. Many works leverage symmetry and group structure in deep learning. For example, in geometric deep learning, different forms of symmetry are incorporated into network architectures ~\citep{bronstein2021geometric}. However, they do not open the black box and explore the algebraic structures of the network itself during training.   

\textbf{Expressibility}. Existing works on expressibility~\citep{li2024chain,liu2022transformers} gives explicit weight construction of neural networks weights (e.g., Transformers) for reasoning tasks like automata, which includes modular addition. However, their works do not discover algebraic structures in the weight space and loss, nor learning dynamics analysis, and it is not clear whether the constructed weights coincide with the actual solutions found by gradient descent, even in synthetic data. 


\textbf{Fourier Bases in Arithmetic Tasks}. Existing works discovered that pre-trained models use Fourier bases for arithmetic operations~\citep{zhou2024pre}. This is true even for a simple Transformer, or even a network with one hidden layer~\citep{morwani2023feature}. Previous works also construct analytic Fourier solutions~\citep{gromov2023grokking} for modular addition, but with the additional assumption of infinite width, unaware of the algebraic structures we discover. Existing theoretical work~\citep{morwani2023feature} also shows group-theoretical results on algebraic tasks related to finite groups, also for networks with one-hidden layers and quadratic activations. Compared to ours, they use the max-margin framework with a special regularization ($L_{2,3}$ norm) rather than $L_2$ loss, do not characterize and leverage algebraic structures in the weight space, and do not analyze the training dynamics. 

\begin{figure*}
    \centering
    \vspace{-0.05in}
    \includegraphics[width=\textwidth]{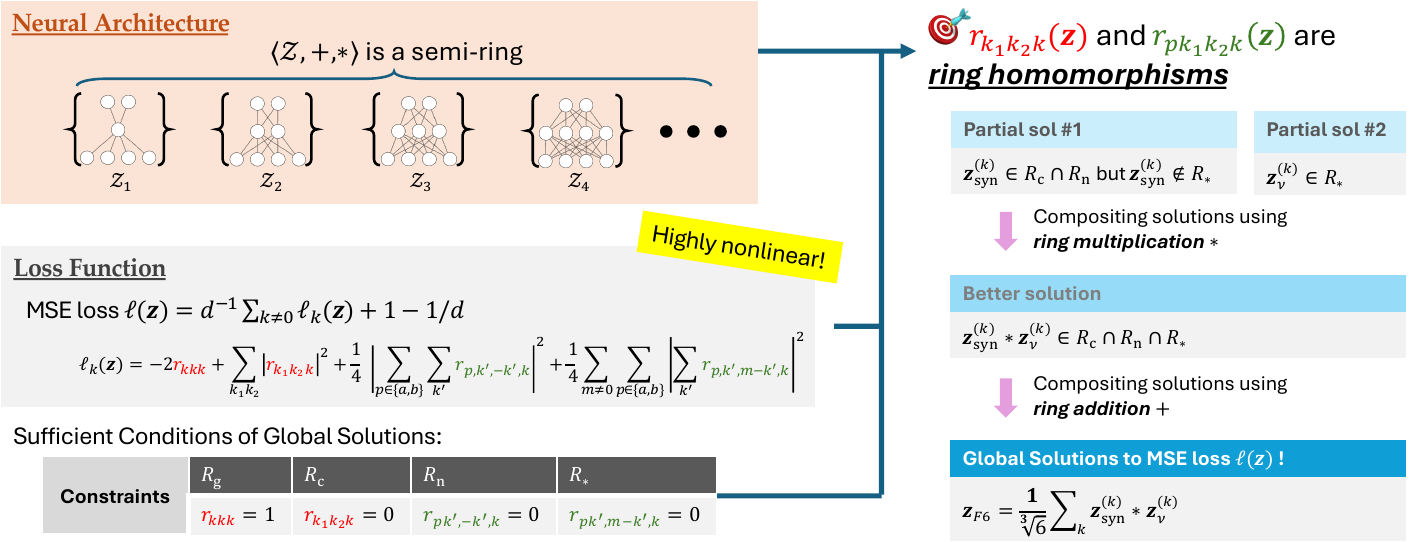}
    \vspace{-0.05in}
    \caption{\small Overview of proposed theoretical framework \ours{}. (1) The family of 2-layer neural networks, $\cZ$, form a \emph{semi-ring} algebraic structure (Theorem~\ref{thm:semi-ring}) with ring addition and multiplication (Def.~\ref{def:operationsinz}). $\cZ = \bigcup_{q\ge 0} \cZ_q$ where $\cZ_q$ is a collection of all weights with order-$q$ (i.e., $q$ hidden nodes). (2) For outcome prediction of Abelian group multiplication, the MSE loss $\ell(\vz)$ is a function of \emph{sum potentials} (SPs) $r_{k_1k_2k}(\vz)$ and $r_{pk_1k_2k}(\vz)$ (Theorem~\ref{thm:analyticform}), which are ring homomorphisms (Theorem~\ref{thm:pothomo}). (3) Thanks to the property of ring homomorphism, global solutions to MSE loss $\ell(\vz)$ with quadratic activation can be constructed \emph{algebraically} from partial solutions that only satisfy a subset of constraints (Sec.~\ref{sec:composing-solutions}) using ring addition and multiplication, instead of running gradient descent. Examples include Fourier solution $\vz_{F6}$ (Corollary~\ref{co:order-6}) and $\vz_{F4/6}$ (Corollary~\ref{co:order-4}) and perfect memorization solution $\vz_M$ (Corollary~\ref{co:perfectmem}). In Sec.~\ref{sec:gradientdynamics}, we analyze the role played of SPs in gradient dynamics, showing that the dynamics favors low-order global solutions (Theorem~\ref{thm:loworderfirst}) under weight decay regularization, and the dynamics of SPs become decoupled with infinite width (Theorem~\ref{lemma:infinitem}). In Sec.~\ref{sec:exp} we show that the gradient descent solutions match exactly with our theoretical construction.} 
\end{figure*}

\vspace{-0.15in}
\section{$L_2$ Loss Decoupling for Abelian group}
\vspace{-0.15in}
\textbf{Basic group theory}. A set $G$ forms a \emph{group}, which means that (1) there exists an operation $\cdot$ (i.e., ``multiplication''): $G\times G \mapsto G$ and it satisfies association: $(g_1\cdot g_2)\cdot g_3 = g_1\cdot( g_2 \cdot g_3)$. Often we write $g_1g_2$ instead of $g_1\cdot g_2$ for brevity. (2) there exists an identity element $e\in G$ so that $eg = ge = g$, (3) for every group element $g\in G$, there is a unique inverse $g^{-1}$ so that $gg^{-1} = g^{-1}g = e$. In some groups, the multiplication operation is commutative, i.e., $gh = hg$ for any $g,h\in G$. Such groups are called \emph{Abelian group}. Modular addition forms a Abelian (more specifically, cyclic) group by noticing that there exists a mapping $a \mapsto e^{2\pi a\i /d}$ and $a+b \mod d$ is $e^{2\pi a\i /d} \cdot e^{2\pi b\i /d} = e^{2\pi (a+b)\i /d}$.

\textbf{Basic Ring theory}. A set $\cZ$ forms a \emph{ring}, if there exists two operations, addition $+$ and multiplication $*$, so that (1) $\langle \cZ, +\rangle$ forms an Abelian group, (2) $\langle \cZ, *\rangle$ is a monoid (i.e., a group without inverse), and (3) multiplication distributes with addition (i.e., $a*(b+c) = a*b+a*c$ and $(b+c)*a = b*a+c*a$). $\cZ$ is called a \emph{semi-ring} if $\langle \cZ, +\rangle$ is a monoid.

\textbf{Notation}. Let $\rr$ be the real field and $\cc$ be the complex field. For a complex vector $\vz$, $\vz^\top$ is its transpose, $\bar\vz$ is its complex conjugate and $\vz^*$ its conjugate transpose. For a tensor $z_{ijk}$, 
$z_{\cdot jk}$ is a vector along its first dimension, $z_{i\cdot k}$ along its second dimension, and $z_{ij\cdot}$ along its last dimension. 

\textbf{Problem Setup}. We consider the following 2-layer networks with $q$ hidden nodes, trained with (projected) $\ell_2$ loss on prediction of group multiplication in Abelian group $G$ with $|G| = d$:
\begin{equation}
    \ell = \sum_i \Big\|P^\perp_1\left(\frac{1}{2d}\vo[i] - \rebut{\ve_{l[i]}}\right)\Big\|^2, \quad
    \vo[i] = \sum_j \vw_{cj} \sigma(\vw_{aj}^\top \ve_{g_1[i]} + \vw_{bj}^\top \ve_{g_2[i]})\label{eq:arch}
\end{equation}
\rebut{
\textbf{Input and Output}. The input contains the two group elements $g_1[i], g_2[i]\in G$ to be multiplied, $\ve_{g_1[i]}, \ve_{g_2[i]} \in \rr^{d}$ are one-hot representation of $g_1[i]$ and $g_2[i]$. Here $i$ is the sample index. The target $\ve_{l[i]}$ is a one-hot representation of $l[i] = g_1[i]g_2[i] \in G$, the group product of $g_1[i]$ and $g_2[i]$.  
}

\textbf{Architectures}. In Eqn.~\ref{eq:arch}, we use quadratic activation $\sigma(x)=x^2$~\citep{du2018power,allen2023backward}, $P_1^\perp = I - \frac{1}{d}\vone\vone^\top$ is the zero-mean projection, \rebut{$\vw_{aj},\vw_{bj}, \vw_{cj} \in \rr^{d}$ are learnable parameters ($1\le j\le q$).} Note that variants of quadratic activation have been used empirically, e.g. squared ReLU and gated activations~\citep{DBLP:journals/corr/abs-2109-08668,shazeer2020glu,zhang2024relu2winsdiscoveringefficient}.  

We can extend our framework to \emph{group action prediction}, in which $g_2$ may not be a group element but any object (e.g., a discrete state in reinforcement learning), See Appendix~\ref{sec:extension-to-group-action-pred}. 

Let $\vphi_k = [\phi_k(g)]_{g\in G} \in \cc^{d}$ be the scaled Fourier bases (or more formally, \emph{character function} of the finite Abelian group $G$, see Appendix~\ref{sec:appendix-decoupling}). Then \rebut{the weight vector $\cW := \{\vw_j\}$} can be written as:
\begin{equation}
    \vw_{aj} = \sum_{k\neq 0} z_{akj} \vphi_k, \quad\vw_{bj} = \sum_{k\neq 0} z_{bkj} \vphi_k, \quad \vw_{cj} = \sum_{k\neq 0} z_{ckj} \bar\vphi_k \label{eq:w-freq-space} 
\end{equation}
where $\vz := \{z_{pkj}\}$ are the complex coefficients, $p \in \{a,b,c\}$, $0\le k < d$ and $j$ runs through $q$ hidden nodes. \rebut{For convenience, we define $\vphi_{-k} := \overline \vphi_k$ as the (complex) conjugate representation of $\vphi_k$}. We exclude $\vphi_0 \equiv 1$ because the constant bias term has been filtered out by the top-down gradient from the loss function. Leveraging the property of quadratic activation functions, we can write down the loss function analytically (see Appendix~\ref{sec:appendix-decoupling}): 
\begin{restatable}[Analytic form of $L_2$ loss with quadratic activation]{theorem}{analyticform}
\label{thm:analyticform}
The objective of 2-layer MLP network with quadratic activation can be written as $\ell = d^{-1}\sum_{k\neq 0} \ell_k + (d-1)/d$, where 
\begin{equation}
    \ell_k = - 2 r_{kkk}\!+\!\sum_{k_1k_2} |r_{k_1k_2k}|^2 + 
    \frac14\Big|\sum_{p\in \{a,b\}} \sum_{k'} r_{p,k',-k',k}\Big|^2 + \frac14\sum_{m\neq 0} \sum_{p\in \{a,b\}} \Big|\sum_{k'} r_{p,k',m-k',k}\Big|^2  
    \label{eq:obj} 
\end{equation}
Here $r_{k_1k_2k} := \sum_j z_{a k_1 j} z_{b k_2 j} z_{ckj}$ and $r_{pk_1k_2k} := \sum_j z_{pk_1j} z_{pk_2j} z_{ckj}$.
\end{restatable}

Note that for cyclic group $G$, the frequency $k$ is a mod-$d$ integer. For general Abelian group which can be decomposed into a direct sum of cyclic groups according to Fundamental Theorem of Finite Abelian Groups~\citep{diaconis1988group}, $k$ is a multidimensional frequency index. Since \rebut{$\{\vw_{pj}\}$} are all real, the Hermitian constraints hold, i.e. $\overline{z_{ckj}} = \rebut{\overline{\vphi^*_k \vw_{cj}} = \vphi^*_{-k} \vw_{cj}} = z_{c,-k,j}$ (and similar for $z_{akj}$ and $z_{bkj}$). Therefore, $z_{p,-k,j} = \bar z_{pkj}$, $r_{-k,-k,-k} = \bar r_{kkk}$ and $\ell$ is real and can be minimized.  

Eqn.~\ref{eq:obj} contains different $r$ terms, which play an important role in determining global solutions. 
\begin{definition}[0/1-set]
Let $\con := \{r\}$ be a collection of $r$ terms. The weight $\vz$ is said to have $0$-set $R_0$ and $1$-set $R_1$ (or 0/1-sets $(R_0,R_1)$), if $r(\vz) = 0$ for all $r\in \con_0$ and $r(\vz) = 1$ for all $r\in \con_1$.
\end{definition}
With 0/1-sets, we can characterize rough structures of the
global solutions to the loss:
\begin{restatable}[A Sufficient Conditions of Global Solutions of Eqn.~\ref{eq:obj}]{lemma}{globalsolutions}
\label{co:globalminimizer}
If the weight $\vz$ to Eqn.~\ref{eq:obj} has 0-sets $R_\c \cup R_\n \cup R_*$ and 1-set $R_\g$, i.e.
\begin{equation}
    r_{kkk}(\vz) = \mathbb{I}(k\neq 0), \quad r_{k_1k_2k}(\vz) = 0, \quad r_{pk_1k_2k}(\vz) = 0  \label{eq:globalminimizer}
\end{equation}
then it is a global solution with $\ell(\vz)\!=\!0$. Here $\con_\g\!:=\!\{r_{kkk}, k\neq 0\}$, $\con_\c := \{r_{k_1k_2k}, k_1, k_2, k \mathrm{\ not\ all\ equal}\}$, $\con_\n := \{r_{p,k',-k',k}\}$ and $\con_* := \{ r_{p,k',m-k',k}, m\neq 0\}$.
\end{restatable}
Lemma~\ref{co:globalminimizer} provides \emph{sufficient} conditions since there may exist solutions that achieve global optimum (e.g., $\sum_{k'} r_{p,k',m-k',k}(\vz) = 0$ but $r_{p,k',m-k',k}(\vz) \neq 0$). However, as we will see, it already leads to rich algebraic structure, and serves as a good starting point. Directly finding the global solutions using Eqn.~\ref{eq:globalminimizer} can be a bit complicated and highly non-intuitive, due to highly nonlinear structure of Eqn.~\ref{eq:obj}. However, there are nice structures we can leverage, as we will demonstrate below. 

\vspace{-0.1in}
\section{Algebraic Property of the Weight Space}
\vspace{-0.1in}
\label{sec:semi-ring}
We define the \emph{weight space} $\cZ_q = \{\vz\}$ to include all the weight matrices with $q$ hidden nodes ($\cZ_0$ means an empty network), and $\cZ = \bigcup_{q\ge 0} \cZ_q$ be the solution space of all different number of hidden nodes. Interestingly, $\cZ$ naturally is equipped with a \emph{semi-ring} structure, and each term of the loss function can effective interact with such a semi-ring structure, yielding provable global solutions, including both the Fourier solutions empirically reported in previous works~\citep{zhou2024pre,gromov2023grokking}, and the perfect memorization solution~\citep{morwani2023feature}. 

To formalize our argument, we start with a few definitions. 

\def\ringadd{+}
\def\ringaddbig{\sum}
\def\ringmul{*}

\begin{definition}[Order of $\vz$]
\label{def:order-weight}
The order $\ord(\vz)$ of $\vz\in \cZ$ is its number of hidden nodes. 
\end{definition}
\begin{definition}[Scalar multiplication]
\label{def:scalar-multiplication}
$\alpha \vz \in \cZ$ is element-wise multiplication $[\alpha z_{pkj}]$ of $\vz\in\cZ$.
\end{definition}
\begin{definition}[Identification of $\cZ$]
\label{def:permutation-invariant}
In $\cZ$, two solutions of the same order that differ only by a permutation along hidden dimension $j$ are considered identical. 
\end{definition}

We define operations for $\vz_1 := \{z^\bone_{pkj}\}$ and $\vz_2 := \{z^\btwo_{pkj}\}$:
\begin{definition}[Addition and Multiplication in $\cZ$]
\label{def:operationsinz}
Define $\vz = \vz_1 \ringadd \vz_2$ in which $z_{pk\cdot} := \mathrm{concat}(z^\bone_{pk\cdot},z^\btwo_{pk\cdot})$ and $\vz = \vz_1 \ringmul \vz_2$, in which $z_{pk\cdot} := z^\bone_{pk\cdot} \otimes z^\btwo_{pk\cdot}$. The addition and multiplication respect Hermitian constraints and the identity element $\vone$ is the $1$-order solutions with $\{z_{pk0} = 1\}$. 
\end{definition}
Note that the multiplication definition is one special case of Khatri–Rao product~\citep{khatri1968solutions}. Although the Kronecker product and concatenation are not commutative, thanks to the identification (Def.~\ref{def:permutation-invariant}), it is clear that $\vz_1 \ringadd \vz_2 = \vz_2 \ringadd \vz_1$ and $\vz_1 * \vz_2 = \vz_2 * \vz_1$ and thus both operations are commutative. Then we can show: 
\begin{restatable}[Algebraic Structure of $\cZ$]{theorem}{zring}
\label{thm:semi-ring}
$\langle\cZ, \ringadd, *\rangle$ is a commutative semi-ring.  
\end{restatable}
As we shall see, Thm.~\ref{thm:semi-ring} allows the construction of global solutions. Now let us explore the structure of the loss (Eqn.~\ref{eq:obj}), which turns out to be connected with the semi-ring structure of $\cZ$. For this, we first define the concept of \emph{sum potentials}: 
\begin{definition}[Sum potential (SP)] 
\label{def:sum-potential}
Let \textbf{\emph{sum potential}} be $r(\vz) := \sum_j \prod_{{p,k}\in \idx(r)} z_{pkj}$ which takes the summation of a monomial term over all hidden nodes. Here $\idx(r)$ specifies the terms involved.
\end{definition}

Following this definition, terms in the loss function (Theorem~\ref{thm:analyticform}) are examples of SPs.
\begin{observation}[Specific SPs]
$r_{k_1k_2k}(\vz)$ and $r_{pk_1k_2k}(\vz)$ defined in Theorem~\ref{thm:analyticform} are SPs.
\end{observation}

\def\fR{\mathfrak{R}}
\def\zz{\mathbb{Z}}
\def\nn{\mathbb{N}}
\def\spann{\mathrm{span}}
\def\vzero{\mathbf{0}}
\def\Ker{\mathrm{Ker}}

So what is the relationship between SPs, which are functions that map a weight $\vz$ to a complex scalar, and the semi-ring structure of $\cZ$? The following theorem tells that SPs are \emph{ring homomorphisms}, that is, these mappings respect addition and multiplication: 
\begin{restatable}{theorem}{pothomo}
\label{thm:pothomo}
For any sum potential $r: \cZ \mapsto \cc$, $r(\vone) = 1$, $r(\vz_1 \ringadd \vz_2) = r(\vz_1) + r(\vz_2)$ and $r(\vz_1 \ringmul \vz_2) = r(\vz_1)r(\vz_2)$ and thus $r$ is a ring homomorphism. 
\end{restatable}
\begin{restatable}{observation}{ordhomo}
\label{lemma:ordhomo}
The order function $\ord: \cZ \mapsto \nn$ is also a ring homomorphism.
\end{restatable}

Since the loss function $\ell(\vz)$ depends on the weight $\vz$ entirely through $r_{k_1k_2k}(\vz)$ and $r_{pk_1k_2k}(\vz)$, which are SPs, due to the property of ring homomorphism, it is possible to construct a global solution from partial solutions that satisfy only some of the constraints\footnote{Mathematically, the \emph{kernel} $\Ker(r) := \{\vz : r(\vz) = 0\}$ of a ring homomorphism $r$ is an \emph{ideal} of the ring, and the intersection of ideals are still ideals. For brevity, we omit the formal definitions.}:
\begin{restatable}[Composing Partial Solutions]{lemma}{composesol}
\label{lemma:composesol}
If $\vz_1$ has 0/1-sets $(R_1^-, R^+_1)$ and $\vz_2$ has 0/1-sets $(R_2^-, R^+_2)$, then (1) $\vz_1*\vz_2$ has 0/1-sets $(R_1^-\cup R_2^-, R^+_1\cap R^+_2)$ and (2) $\vz_1 + \vz_2$ have 0/1-sets $(R_1^-\cap R_2^-, (R^+_1\cap R^-_2)\cup (R^-_1 \cap R^+_2))$. 
\end{restatable} 
Once we reach 0/1-sets $(R_\c \cup R_\n \cup R_*, R_\g)$, we find a global solution. In addition, we also immediately know that there exists infinitely many global solutions, via ring multiplication (Def.~\ref{def:operationsinz}):
\begin{definition}[Unit] $\vz$ is a \emph{unit} if $r_{kkk}(\vz) = 1$ for all $k \neq 0$. 
\end{definition}
\begin{restatable}{corollary}{algebraglobaloptimizer}
\label{co:algebraoptimizer}
If $\vz$ is a global solution and $\vy$ is a unit, then $\vz \ringmul \vy$ is also a global solution.
\end{restatable} 


\vspace{-0.1in}
\section{Composing Global Solutions}
\vspace{-0.1in}
\label{sec:composing-solutions}
\textbf{Constructing Partial Solutions with Polynomials}. While intuitively one can get global solutions by manually crafting some partial solutions and combining, in this section, we provide a more systematic approach to compose global solutions as follows. Since $\cZ$ enjoys a semi-ring structure, we consider a \emph{polynomial} in $\cZ$: 
\begin{equation}
    \vz = \vu^{L} \ringadd \vc_1 \ringmul \vu^{L-1} \ringadd \vc_2 \ringmul \vu^{L-2} \ringadd \ldots \ringadd \vc_L
\end{equation}
where the \emph{generator} $\vu$ and coefficients $\vc_l$ are order-1 and the power operation $\vu^l$ is defined by ring multiplication. Then a partial solution can be constructed: 

\begin{restatable}[Construction of partial solutions]{theorem}{polyconst}
\label{thm:polyconstruct}
Suppose $\vu$ has 1-set $\con_1$, $\Omega_{\con}(\vu) := \{r(\vu) | r \in \con \} \subseteq \cc$ is a set of evaluations on $\con$ (multiple values counted once), then if $1 \notin \Omega_\con$, then the polynomial solution $\vrho_R(\vu) := \prod_{s\in \Omega_\con(\vu)} (\vu + \hat\vs)$ has 0/1-set $(R, R_1)$ up to a scale. Here $\hat\vs$ is any order-1 weight that satisfies $r(\hat\vs) = -s$ for any $r \in R \cup R_1$. For example, $\hat\vs = -s^{1/3}\vone$.
\end{restatable}
For convenience, we use $\vrho(\vu)$ to represent the \emph{maximal} polynomial, i.e., when $R = \arg\max_{1\notin\Omega_R(\vu)} |\Omega_R(\vu)|$ is the largest subset of SPs with $1\notin\Omega_R(\vu)$. Our goal is to find low-order (partial) solutions, since gradient descent prefers low order solutions (see Theorem~\ref{thm:loworderfirst}). Although there exist high-degree but low-order polynomials, e.g., $\vu^9+\vone$, in general, degree $L$ and order $q$ are correlated, and we can find low-degree ones instead. To achieve that, $\vu$ should be properly selected (e.g., symmetric weights) to create as many duplicate values (but not $1$) in $R$ as possible. 

\begin{table}[t]
\small
    \centering
    \setlength{\tabcolsep}{2pt}

\begin{tabular}{c|c||c|c|c||c|c||c|c|c|c||c|c}
        &  & \multicolumn{9}{c|}{Evaluation on SPs} &  &  \\ 
                       &  & \multicolumn{3}{c|}{$R_\c$} & \multicolumn{2}{c|}{$R_\n$} & \multicolumn{4}{c|}{$R_*$} & Maximal & \\
        Symbol & $[a, b, c]$ &  $\bar a b c$ & $a \bar b c$ & $ab\bar c$ & $\bar a a c$ & $\bar b b c$ & $aac$ & $bbc$ & $\bar a\bar a c$ & $\bar b\bar b c$ & polynomial $\vrho(\vu)$ & order $q$ \\
                       
       \hline
       $\vone_k$ & $[1, 1, 1]$ & \cellcolor{red!25} $1$ & \cellcolor{red!25} $1$ & \cellcolor{red!25} $1$ & \cellcolor{red!25} $1$ & \cellcolor{red!25} $1$ & \cellcolor{red!25} $1$ & \cellcolor{red!25} $1$ & \cellcolor{red!25} $1$ & \cellcolor{red!25} $1$ & -- & -- \\ 
       
       $\tilde\vone_k$ & $[-1, -1, 1]$ &\cellcolor{red!25} $1$ &\cellcolor{red!25} $1$\cellcolor{red!25} &\cellcolor{red!25} $1$ &\cellcolor{red!25} $1$ &\cellcolor{red!25} $1$ &\cellcolor{red!25} $1$ &\cellcolor{red!25} $1$
       &\cellcolor{red!25} $1$ &\cellcolor{red!25} $1$ & -- & -- \\ 
       
      $\vu_{\one}$  & $[1, -1, -1]$ &\cellcolor{red!25} $1$ &\cellcolor{red!25} $1$\cellcolor{red!25} &\cellcolor{red!25} $1$ & $-1$ &$-1$ &$-1$ &$-1$ &$-1$ &$-1$ & $\vu + \vone$ & 2 \\ 
       
       $\vu_{\syn}$ & $[\omega_3, \omega_3, \omega_3]$  & $\omega_3$ & $\omega_3$ & $\omega_3$ & $\omega_3$ & $\omega_3$ &\cellcolor{red!25} $1$ &\cellcolor{red!25} $1$ & $\bar\omega_3$ & $\bar\omega_3$ & $\vu^2 + \vu + \vone$ & 3 \\
       
       $\vu_{3\c}$ & $[\omega_3, \bar\omega_3, 1]$  & $\omega_3$ & $\bar\omega_3$ &\cellcolor{red!25} $1$ &\cellcolor{red!25} $1$ &\cellcolor{red!25} $1$ & $\bar\omega_3$ & $\omega_3$ & $\omega_3$ & $\bar\omega_3$ & $\vu^2 + \vu + \vone$ & 3 \\

       $\vu_{3\mathrm{a}}$ & $[1, \omega_3, \bar\omega_3]$  & \cellcolor{red!25} $1$ & $\omega_3$ 
       & $\bar\omega_3$ & 
       $\bar\omega_3$ & $\bar\omega_3$ & $\bar\omega_3$ & $\omega_3$ & $\bar\omega_3$ &\cellcolor{red!25} $1$ & $\vu^2 + \vu + \vone$ & 3 \\

       $\vu_{4\mathrm{c}}$ & $[\i, -\i, 1]$  & $-1$ & $-1$ & \cellcolor{red!25} $1$ &\cellcolor{red!25} $1$  &\cellcolor{red!25} $1$ & $-1$ & $-1$ & $-1$ & $-1$ & $\vu + \vone$ & 2 \\

       $\vu_{4\mathrm{a}}$ & $[1, \i, -\i]$  &\cellcolor{red!25} $1$ & $-1$ & $-1$ & $-\i$ & $-\i$ & $-\i$ & $\i$ & $-\i$ & $\i$ & $\vu^3 + \vu^2 + \vu + \vone$ & 4 \\ 
       
       $\vu_{\nu}$ & $[\nu, -\nu, -\bar\nu^2]$ & $\nu^2$ & $\nu^2$ & $\nu^4$ & $-\bar \nu^2$ & $-\bar \nu^2$ & $-1$ & $-1$ & $-\nu^4$ & $-\nu^4$ & 9-th degree & 10  
\end{tabular}
    \caption{\small Exemplar order-1 single frequency generator $\vu^\bk$ with $r_{kkk}(\vu^\bk) = 1$. In the single-frequency case, for each MP $r$ we use ``$\bar abc$'' to represent $r_{-k,k,k}$ and ``$\bar a\bar a c$'' to represent $r_{a,-k,-k,k}$, etc. For brevity, superscript ``$\bk$'' and conjugate columns (i.e., $\bar a \bar b c$ conjugate to $a b \bar c$) are omitted. Here, $\omega_3 := e^{2\pi \i / 3}$ and $\omega_4 := \i$ are 3rd/4th roots of unity. The constructions are partial, i.e., the evaluation of some SPs yields $1$ ({\color{red}red} cell) and cannot be the root of the polynomial (Theorem~\ref{thm:polyconstruct}). Note that $\vu_\nu$ is a general case with $\vu_{\nu=1} = \vu_\one$ and $\vu_{\nu=\i} = \vu_{4\c}$.}
    \label{tab:poly-construction}
\end{table}

\textbf{Composing Global Solutions}. We first consider the case that the generator $\vu$ is only nonzero at frequency $k$ (and thus $-k$ by Hermitian constraints), but zero in other frequencies, i.e., $u_{pk'0} = 0$ for $k' \neq \pm k$. Such solutions correspond to Fourier bases in the original domain. Also, $\vu$ has 1-set $R_1 = \{r_{kkk}\}$. This means that $\vu$ can be characterized by three numbers $u_{ak0}, u_{bk0}, u_{ck0}$ with $u_{ak0}u_{bk0}u_{ck0} = 1$. In this case, only a subset of sum potentials (SPs) whose indices only involve a single frequency $k$ are non-zero (e.g., $r_{k,-k,k} \in R_\c$ and $r_{b,-k,k,k} \in R_\n$), facilitating our construction. 

Following Theorem~\ref{thm:polyconstruct}, we can construct different partial solutions. Some examples are shown in Table~\ref{tab:poly-construction}. These solutions do not make all sum potentials in $R_\c \cup R_\n \cup R_*$ vanish and therefore are not global. Note that it is possible to create a global solution this way, but then $|\Omega_\con(\vu)|$ will be too large, producing high-degree/order polynomials (e.g., $\vu_{3\c} \ringmul \vu_{4\a}$ gives a 10th-degree polynomial). Instead, utilizing these partial solutions, with Lemma~\ref{lemma:composesol} we can construct global solutions with smaller orders:  

\begin{restatable}[Order-6 global solutions]{corollary}{ordersix}
\label{co:order-6}
The following $``3\times 2"$ Fourier solutions satisfy the sufficient condition  (Lemma~\ref{co:globalminimizer}) and thus are global solutions when $d$ is odd:
\begin{equation}
    \vz_{F6} = \frac{1}{\sqrt[3]{6}}\ringaddbig_{k=1}^{(d-1)/2} \vz^\bk_{\syn} \ringmul \vz^\bk_{\nu} \ringmul \vy_k
\end{equation}
Here $\vz^\bk_{\syn} := \vrho(\vu^\bk_\syn)$ and $\vz^\bk_{\nu} := \vu^\bk_{\nu} + \vone_k$ (i.e., not maximal polynomial), where $\vu_{\syn}$ and $\vu_{\nu}$ are defined in Table~\ref{tab:poly-construction}. $\vy$ is an order-1 unit. As a result, $\ord(\vz_{F6}) = 3 \cdot 2 \cdot 1 \cdot (d-1)/2 = 3(d-1)$ and each frequency are affiliated with 6 hidden nodes (order-6).  
\end{restatable}
\underline{\emph{Remarks}}. We may replace $\vu_\syn$ and $\vu_{\nu}$ with other pairs that collectively cover all SPs. For example, $\vu_\syn$ can be combined with any of $\{\vu_{3\c}, \vu_{3\a}, \vu_{4\a}\}$, and $\vu_{\nu=\pm \i}$ can be coupled with $\vu_{3\a}$ or $\vu_{4\a}$, etc. Here we pick one with a small order. Compared to~\cite{gromov2023grokking}, our construction is more concise without infinite-width approximation.  
For even $d$, simply replace $(d-1)/2$ with $\lfloor(d-1)/2\rfloor$ and add an additional order-2 term $\vrho(\vu_\one) = \vu_\one + \vone$ (Tbl.~\ref{tab:poly-construction}) for frequency $k=d/2$, which only has $r_{kkk}$, $r_{akkk}$ and $r_{bkkk}$, and all other combinations are absent. 

\begin{figure}
    \centering
    \vspace{-0.15in}
    \includegraphics[width=0.9\linewidth]{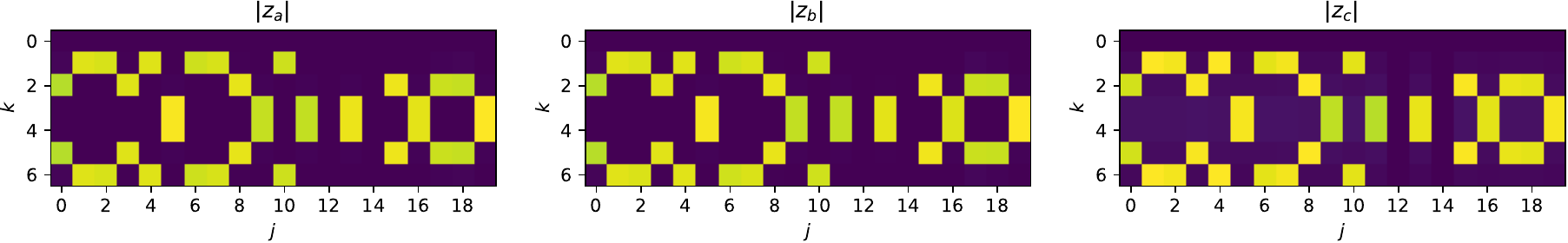}
    \includegraphics[width=0.9\linewidth]{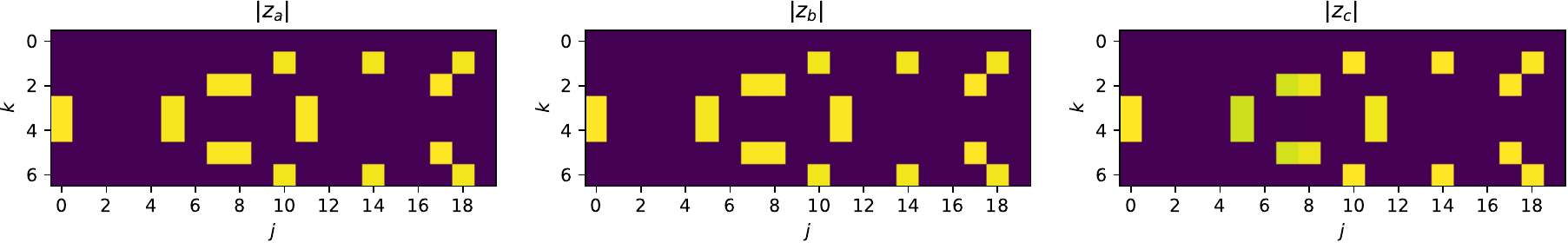}
    \vspace{-0.15in}
    \caption{\small Solutions obtained by Adam optimizers on $\ell_2$ loss for modular addition task with $|G|\!=\! d\! =\! 7$ and $q\! =\! 20$ hidden nodes. \textbf{Top:} For each frequency $\pm k$, exactly $6$ hidden nodes exist (Corollary~\ref{co:order-6}). \textbf{Bottom:} Optimizing Eqn.~\ref{eq:obj} without the last term $\sum_{m\neq 0} \sum_{p\in \{a,b\}} \Big|\sum_{k'} r_{p,k',m-k',k}\Big|^2$ (i.e., without constraint $R_\conv$). Now each frequency has exactly $3$ hidden nodes, corresponding to the solution $\vz_{\syn} = \vrho(\vu_{\syn})$ in Tbl.~\ref{tab:poly-construction}.} 
    \label{fig:solution-structure}
\end{figure}

\def\ba{(\alpha)}
\def\bb{(\beta)}

Fig.~\ref{fig:solution-structure} shows a case with $d=7$. In this case, each frequency, out of $(d-1)/2 = 3$ total number of frequencies, is associated with $6$ hidden nodes. If we remove the last term in the loss that corresponds to $R_\conv$, then an order-3 solution suffices (i.e. $\vz_\syn = \vrho(\vu_{\syn})$). Perfect-memorization solutions can also be constructed. Let two generators be $\vu_\alpha$ with $u^{\ba}_{\cdot k0}\!=\![\omega_d^k, 1, \bar\omega_d^k] \mathbb{I}(k\!\neq\! 0)$, and $\vu_\beta$ with $u^{\bb}_{\cdot k0}\!=\![1, \omega_d^k, \bar\omega_d^k] \mathbb{I}(k\!\neq\! 0)$. Here $\omega_d := e^{2\pi\i / d}$ is the $d$-th root of unity. Then, 

\begin{restatable}[Perfect Memorization]{corollary}{perfectmem}
\label{co:perfectmem}
We construct two $d$-order weights $(\vz_\alpha, \vz_\beta)$ from order-1 generators $(\vu_\alpha, \vu_\beta)$: 
\begin{equation}
    \vz_\alpha = \ringaddbig_{j=0}^{d-1} \vu_\alpha^j, \quad\quad \vz_\beta = \ringaddbig_{j=0}^{d-1} \vu_\beta^j 
\end{equation}
Here $\vz_\alpha \in R_\c(k_1 \neq k) \cap R_\n \cap R_\conv(p = b \mathrm{\ or\ } m \neq k)$, $\vz_\beta \in R_\c(k_2 \neq k) \cap R_\n \cap R_\conv(p = a \mathrm{\ or\ } m \neq k)$. Then $\vz_M = d^{-2/3}\vz_\alpha * \vz_\beta$ satisfies the sufficient condition (Lemma~\ref{co:globalminimizer}) and is the perfect memorization solution with $\ord(\vz_M) = d^2$:
\begin{equation}
    z^{(M)}_{akj_1j_2} = d^{-\frac{2}{3}}\omega^{kj_1}, \quad z^{(M)}_{bkj_1j_2} = d^{-\frac{2}{3}}\omega^{kj_2}, \quad z^{(M)}_{ckj_1j_2} = d^{-\frac{2}{3}} \omega^{-k(j_1+j_2)} 
\end{equation}
where each hidden node is indexed by $j = (j_1,j_2)$, $0\le j_1,j_2< d$, $k\neq 0$. 
\end{restatable}
To see why this corresponds to perfect memorization, simply apply an inverse Fourier transform for each hidden node $(j_1,j_2)$, which leads to (zero-mean) delta weights located at $j_1$, $j_2$ and $j_1+j_2$. Interestingly, there also exists a lower-order solution, $2\times 2$, that meets $R_\c$ and $R_\conv$ but not $R_\n$:
\begin{restatable}[Order-4 single frequency solution]{corollary}{orderfour}
\label{co:order-4}
Define single frequency order-2 solution $\vz_\xi$:
\begin{equation}
    z_{ak\cdot} = [1, \xi],\quad z_{bk\cdot} = [1, -\i\bar\xi],\quad z_{ck\cdot} = [1, \i]  
\end{equation}
where $|\xi|=1$. Then the order-4 solution $\vz^\bk_{F4} := \vrho(\vu^\bk_{\nu=\i}) * \vz_\xi^\bk$ has 0-sets $R_\c$ and $R_\conv$ (but not $R_\n$). 
\end{restatable}
Although $\vz^\bk_{F4}$ does not satisfy the sufficient condition (Eqn.~\ref{eq:globalminimizer}), it is part of a global solution when mixed with $\vz_{F6}$:

\begin{restatable}[Mixed order-4/6 global solutions]{corollary}{foursixsol}
\label{co:foursixsol}
With $\vz^\bk_{F4}$, there is a global solution to Eqn.~\ref{eq:obj} that does not meet the sufficient condition, i.e., $\sum_{k'} r_{p,k',-k',m} = 0$ but $r_{p,k',-k',m} \neq 0$:
\begin{equation}
    \vz_{F4/6} = \frac{1}{\sqrt[3]{6}}\hat \vz^{(k_0)}_{F6} \ringadd \frac{1}{\sqrt[3]{4}} \ringaddbig_{k=1,k\neq k_0}^{(d-1)/2} \vz^\bk_{F4} 
\end{equation}
where $\hat \vz^{(k_0)}_{F6}$ is a perturbation of $\vz^{(k_0)}_{F6} := \vz^{(k_0)}_{\syn} \ringmul \vz^{(k_0)}_{\nu=1}$ by adding constant biases to its $(c, k)$ entries for $k\neq k_0$. The order is lower than $\vz_{F6}$: $\ord(\vz_{F4/6}) = 6 + 4 \cdot ((d-1)/2 - 1) = 2d < \ord(\vz_{F6})$. 
\end{restatable}
The specific formats of $\hat \vz^{(k_0)}_{F6}$ is shown in Appendix (please check the proof and Eqn.~\ref{eq:synab}). Multiple order-6 solutions per frequency can be inserted in this construction. Compared to $\vz_{F6}$, this order-4/6 mixture solution has a lower order and is perceived in the experiments (See Fig.~\ref{fig:convergence-path}), in particular when $d$ is large (Tbl.~\ref{tab:factorization}), showing a strong preference of gradient descent towards lower order solutions.

\def\dd{\mathrm{d}}

\begin{figure*}
\centering
    \includegraphics[width=0.32\textwidth]{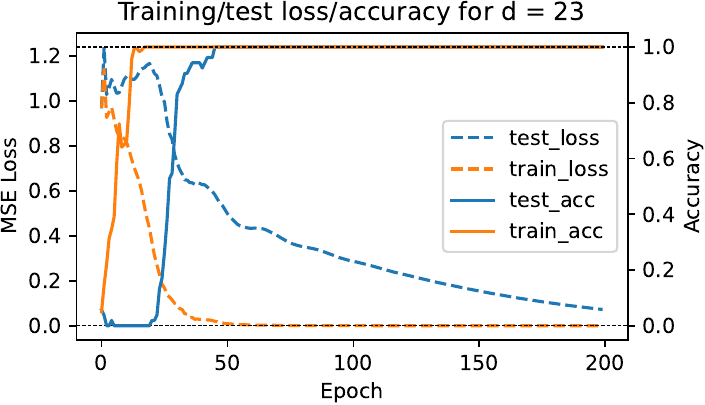}\hfill
    \includegraphics[width=0.32\textwidth]{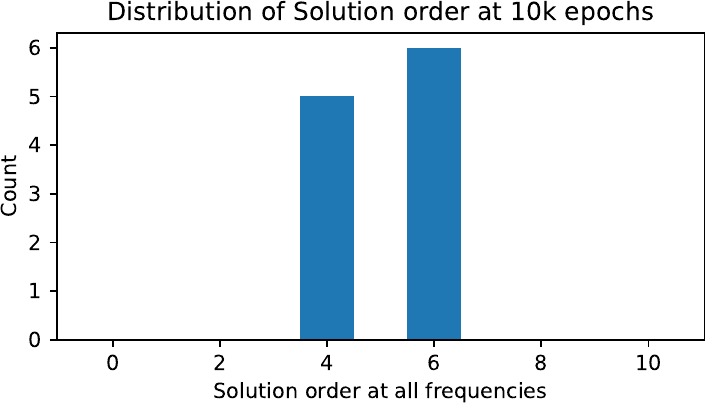}\hfill
    \includegraphics[width=0.32\textwidth]{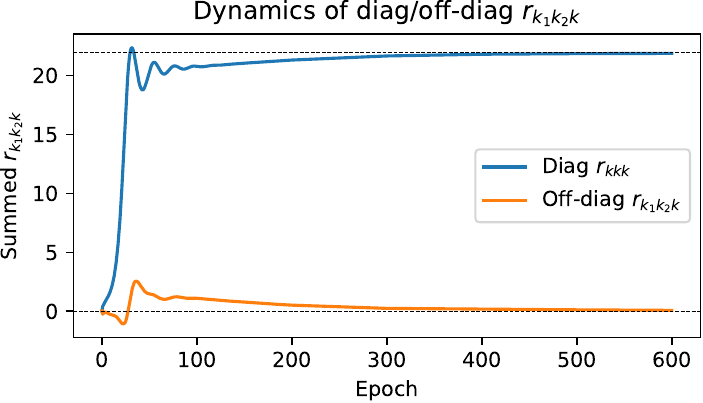}
    \includegraphics[width=0.32\textwidth]{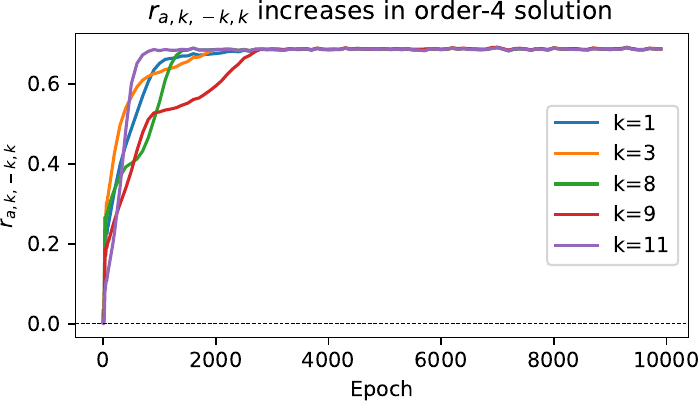}\hfill
    \includegraphics[width=0.32\textwidth]{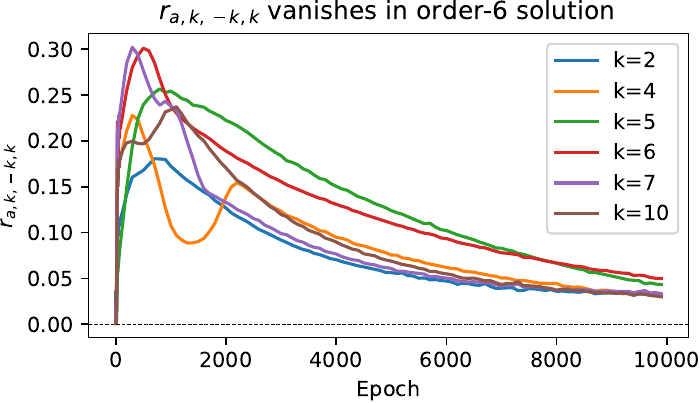}\hfill
    \includegraphics[width=0.32\textwidth]{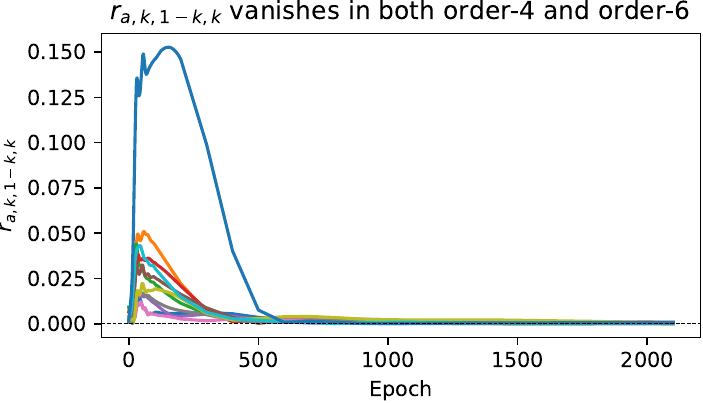}
    \vspace{-0.15in} 
    \caption{\small Dynamics of sum potentials (SPs) over the training process for modular addition with $d = 23$ and $q = 1024$ hidden nodes. \textbf{Top Row.} \emph{Left}: Training/test accuracy reaches 100\% and loss close to $0$. Test accuracy jumps after training reaches 100\% (grokking). \emph{Mid}: After 10k epochs, the distribution of solution orders are concentrated at 4 and 6 (Corollary~\ref{co:order-6} and \ref{co:order-4}). \emph{Right}: Dynamics of $r_{k_1k_2k}$. Summation of diagonal $r_{kkk}$ converges towards $d-1$ (dotted line) with ripple effects, while off-diagonal $r_{k_1k_2k}$ converges towards $0$. \textbf{Bottom Row.} Dynamics of different SPs. Order-4 and order-6 behave differently on $r_{p,k,-k,k}$, because order-4 does not satisfy the sufficient condition (Lemma~\ref{co:globalminimizer}) but a mixture of order-4 and order-6 (i.e., $\vz_{F4/6}$) is still the global solution to the $L_2$ loss (Corollary~\ref{co:foursixsol}).}
    \label{fig:dynamics-mps}
\end{figure*}

\vspace{-0.1in}
\section{Exploring the solution solution with Gradient dynamics}
\vspace{-0.1in}
\label{sec:gradientdynamics}
Now we have characterized the structures of global solutions. One natural question arises: why does the optimization procedure not converge to the perfect memorization solution $\vz_M$, but to the Fourier solutions $\vz_{F6}$ and $\vz_{F4/6}$? Although characterizing the full gradient dynamics is beyond the scope of this paper, we theoretically characterize some rough behaviors below. 

Corollary~\ref{co:algebraoptimizer} shows that by ring multiplication, we could create infinitely many global solutions from one. Then Thm.~\ref{thm:loworderfirst} answers which solution the gradient dynamics may pick: 
\begin{restatable}[The Occam's Razer: Preference of low-order solutions]{theorem}{loworder}
\label{thm:loworderfirst}
If $\vz = \vy \ringmul \vz'$ and both $\vz$ (of order $q$) and $\vz'$ are global optimal solutions, then there exists a path of zero loss connecting $\vz$ and $\vz'$ in the space of $\cZ_q$. As a result, lower-order solutions are preferred if trained with $L_2$ regularization. 
\end{restatable}
This shows that gradient dynamics (with weight decay) may pick a lower-order (i.e., simpler) solution. This suggests that gradient dynamics may not favor perfect memorization, which is of high order. We leave it a future work to prove the existence of a path that connects perfect memorization solutions with a lower-order one. The following theorem shows that the dynamics enjoys \emph{asymptotic freedom}: 
\begin{restatable}[Infinite Width Limits at Initialization]{theorem}{infinitelimit}
\label{lemma:infinitem}
Considering the modified loss of Eqn.~\ref{eq:obj} with only the first two terms: $\tilde \ell_k := -2r_{kkk} + \sum_{k_1k_2} |r_{k_1k_2k}|^2$, if the weights are i.i.d Gaussian and network width $q \rightarrow +\infty$, then $JJ^*$ converge to diagonal and the dynamics of SPs is decoupled.  
\end{restatable}
Intuitively, this means that a large enough network width ($q\rightarrow +\infty$) makes the dynamics much easier to analyze. On the other hand, the final solution may not require that large $q$. As analyzed in Corollary~\ref{co:order-6}, for each frequency, to achieve global optimality, $6$ hidden nodes suffice. 

\begin{figure}
    \vspace{-0.1in} 
    \centering
    \includegraphics[width=0.9\textwidth]{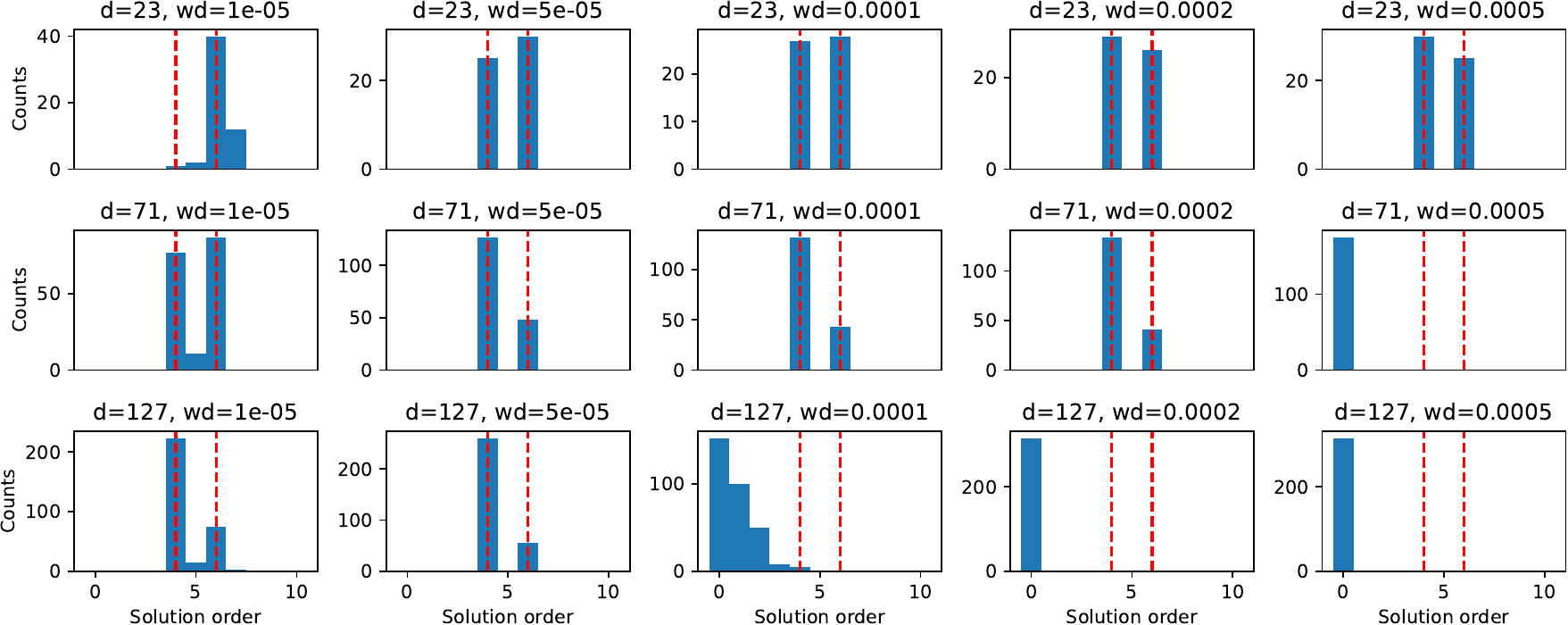}
    \vspace{-0.1in} 
    \caption{\small Solution distribution (accumulated over 5 random seeds) over different weight decay regularization for $q = 512$, trained with 10k epochs with Adam with learning rate $0.01$ on modular addition (i.e., predicting $a+b\mod d$) with $d\in \{23,71,127\}$. Red dashed lines correspond to order-4/6 solutions.}
    \label{fig:low-order-first}
\end{figure}

\vspace{-0.15in}
\section{Experiments}
\vspace{-0.15in}
\label{sec:exp}
\textbf{Setup}. We train the 2-layer MLP on the modular addition task, which is a special case of outcome prediction of Abelian group multiplication. We use Adam optimizer with learning rate $0.01$, MSE loss, and train for $10000$ epochs with weight decays. We tested on $|G| = d \in \{23, 71, 127\}$. All data are generated synthetically and training/test split is $90\%/10\%$. Each training with a fixed set of hyperparameter configuration is conducted on NVIDIA V100 for a few minutes.  

\textbf{Solution Distributions}. As shown in Fig.~\ref{fig:dynamics-mps}, we see order-4 and order-6 solutions in each frequency emerging from well-trained networks on $d=23$. The mixed solution $\vz_{F4/6}$ can be clearly observed in a small-scale example (Fig.~\ref{fig:convergence-path}). This is also true for larger $d$ (Fig.~\ref{fig:low-order-first}). Although the model is trained with heavily over-parameterized networks, the final solution order remains constant, which is consistent with Corollary~\ref{co:globalminimizer}. Large weight decay shifts the distribution to the left (i.e., low-order solutions) until model collapses (i.e., all weights become zero), consistent with our Theorem~\ref{thm:loworderfirst} that demonstrates that gradient descent with weight decay favors low-order solutions. Similar conclusions follow for fewer and more overparameterization (Appendix~\ref{sec:appendix-additional-exp}).

\begin{table*}[]
    \small
    \centering
    \setlength{\tabcolsep}{1pt}
    \vspace{-0.2in}
    \begin{tabular}{c||c|c|c||c|c||c|c|c|c}
      \multirow{2}{*}{$d$} & \%not  & \multicolumn{2}{c||}{\%non-factorable} & \multicolumn{2}{c||}{error ($\times 10^{-2}$)} & \multicolumn{4}{c}{solution distribution (\%) in factorable ones} \\ 
       & order-4/6 & order-4 & order-6 & order-4 & order-6 & $ \vz^\bk_{\nu=\i} \ringmul \vz^\bk_\xi$ & $ \vz^\bk_{\nu=\i} \ringmul \vz^\bk_{\syn,\alpha\beta}$ & $ \vz_{\nu}^\bk \ringmul \vz^\bk_\syn$ & others \\
     \hline\hline
     23 & $0.0${\tiny$\pm 0.0$}& $0.00${\tiny$\pm 0.00$}& $5.71${\tiny$\pm 5.71$}& $0.05${\tiny$\pm 0.01$}& $4.80${\tiny$\pm 0.96$}& $47.07${\tiny$\pm 1.88$}& $11.31${\tiny$\pm 1.76$}& $39.80${\tiny$\pm 2.11$}& $1.82${\tiny$\pm 1.82$} \\
     71 & $0.0${\tiny$\pm 0.0$}& $0.00${\tiny$\pm 0.00$}& $0.00${\tiny$\pm 0.00$}& $0.03${\tiny$\pm 0.00$}& $5.02${\tiny$\pm 0.25$}& $72.57${\tiny$\pm 0.70$}& $4.00${\tiny$\pm 1.14$}& $21.14${\tiny$\pm 2.14$}& $2.29${\tiny$\pm 1.07$} \\
    127 & $0.0${\tiny$\pm 0.0$}& $1.50${\tiny$\pm 0.92$}& $0.00${\tiny$\pm 0.00$}& $0.26${\tiny$\pm 0.14$}& $0.93${\tiny$\pm 0.18$}& $82.96${\tiny$\pm 0.39$}& $2.25${\tiny$\pm 0.64$}& $14.13${\tiny$\pm 0.87$}& $0.66${\tiny$\pm 0.66$}
    \end{tabular}
    \vspace{-0.05in}
    \caption{\small Matches between order-4/6 solutions from gradient descent and those constructed by \ours{}. Number of hidden nodes $q = 512$ and weight decay is $5\times 10^{-5}$. Around $95\%$ gradient descent  solutions are factorable with very small factorization error ($\sim 0.04$ compared to solution norm on the order of $1$). Furthermore, \ours{} successfully predicts $\sim 98\%$ of the structure of the empirical solutions, while the remaining $2\%$ are largely due to insufficient training, near miss against known theoretical construction. Here $\vz_\xi$ is defined in Corollary~\ref{co:order-4}, $\vz_{\nu} := \vu_{\nu} + \vone$ is defined in Tbl.~\ref{tab:poly-construction}, and $\vz_{\syn,\alpha\beta}$ is defined in Eqn.~\ref{eq:synab}. The means/standard deviations are computed over 5 seeds.}
    \label{tab:factorization}
\end{table*}

\textbf{Exact match between theoretical construction and empirical solutions}. A follow-up question arises: \emph{do the empirical solutions match exactly with our constructions?} After all, distribution of solution order is a rough metric. For this, we identify all solutions obtained by gradient descent at each frequency, factorize them and compare with theoretical construction up to conjugation/normalization. To find such a factorization, we use exhaustive search (Appendix~\ref{sec:appendix-additional-exp}). 

The answer is yes. Tbl.~\ref{tab:factorization} shows that around $95\%$ of order-4 and order-6 solutions from gradient descent can be factorized into $2\times 2$ and $2\times 3$ and each component matches our theoretical construction in Corollary~\ref{co:order-6} and~\ref{co:order-4}, with minor variations. Furthermore, when $d$ is large, most of the solutions become order-4, which is consistent with our analysis for mixed solution $\vz_{F4/6}$ (Corollary~\ref{co:foursixsol}) that one order-6 solution in the form of $\vz_{\nu=\i} \ringmul \vz_{\syn,\alpha\beta}$ suffices to achieve a global solution, with all other frequencies taking order-4s. In fact, for $d=127$, the number of order-6 solution taking the form of $\vz_{\nu=\i} \ringmul \vz_{\syn,\alpha\beta}$ is $(d - 1) / 2 \cdot 2.25\% \approx 1.26$, coinciding with the theoretical results.

\textbf{Implicit Bias of gradient descent}. Our construction gives other possible solutions (e.g., $\vz_{3\c} \ringmul \vz_\syn $) which are never observed in the gradient solutions. Even for the observed solutions, e.g. $\vz_{\nu} \ringmul \vz_\syn$, the distribution of free parameters is highly non-uniform (see Fig.~\ref{fig:distri-params} in Appendix), showing a strong preference of parameters that lead to symmetry. These suggest strong implicit bias in optimization, which we leave for future work. 

\vspace{-0.15in}
\section{Conclusion and future work}
\vspace{-0.15in}
In this work, we propose \ours{} (\emph{Composing Global Solutions}), a theoretical framework that models the algebraic structure of global solutions when training a 2-layer network on reasoning tasks of Abelian group with $L_2$ loss. We find that the global solutions can be algebraically composed by partial solutions that only fit parts of the loss, using ring operations defined in the weight space of the 2-layer neural networks across different network widths. Under \ours{}, we also analyze the training dynamics, show the benefit of over-parameterization, and the inductive bias towards simpler solutions due to topological connectivity between algebraically linked high-order (i.e., involving more hidden nodes) and low-order global solutions. Finally, we show that the gradient descent solutions exactly match what constructed solutions (e.g. $\vz_{F4/6}$ and $\vz_{F6}$, see Corollary~\ref{co:foursixsol} and Corollary~\ref{co:order-6}).

\textbf{Develop novel training algorithms}. Instead of applying (stochastic) gradient descent to  overparameterized networks, \ours{} suggests a completely different path: decompose the loss, find the SPs, construct low-order solutions and combine them to achieve the final solutions on the fly using algebraic operations. Such an approach may be more efficient and scalable than gradient descent, due to its factorable nature. Also, our framework works for losses depending on sum potentials ($L_2$ loss is just one example), which opens a new dimension for loss design.  

\textbf{Putting different widths into the same framework}. Many existing theoretical works study properties of networks with fixed width. However, \ours{} demonstrates that nice mathematical structures emerge when putting networks of different widths together, which is an interesting direction to consider. This is related to dynamically adding/pruning neurons during training~\cite{yoon2017lifelong,yu2018slimmable,wu2019splitting}.

\textbf{Grokking}. When learning modular addition, there 
exists a phase transition from \emph{memorization} to \emph{generalization} during training, known as \emph{grokking}~\citep{varma2023explaining,power2022grokking}, long after the training performance becomes (almost) perfect. While our work does not directly address grokking, which involves more complicated training dynamics than described in Sec.~\ref{sec:gradientdynamics}, our framework may be extended to a nonuniformly distributed training set (e.g. some input pairs $(g_1,g_2)$ are missing in the training set), in order to study the dynamics of representation learning on grokking. 

\textbf{Extending to other activations and loss functions.} For other activations (e.g., SiLU) with $\sigma(0) = 0$, with a Taylor expansion, the same framework may still apply, but with higher rank sum potentials (SPs). For other loss functions, we can do a similar Taylor expansion. We leave them for future work.

\bibliography{references}

\begin{thebibliography}{44}
\providecommand{\natexlab}[1]{#1}
\providecommand{\url}[1]{\texttt{#1}}
\expandafter\ifx\csname urlstyle\endcsname\relax
  \providecommand{\doi}[1]{doi: #1}\else
  \providecommand{\doi}{doi: \begingroup \urlstyle{rm}\Url}\fi

\bibitem[Allen-Zhu \& Li(2023)Allen-Zhu and Li]{allen2023backward}
Zeyuan Allen-Zhu and Yuanzhi Li.
\newblock Backward feature correction: How deep learning performs deep (hierarchical) learning.
\newblock In \emph{The Thirty Sixth Annual Conference on Learning Theory}, pp.\  4598--4598. PMLR, 2023.

\bibitem[Anthropic()]{claude3}
Anthropic.
\newblock The claude 3 model family: Opus, sonnet, haiku.
\newblock URL \url{https://www.anthropic.com/news/claude-3-family}.

\bibitem[Berglund et~al.(2023)Berglund, Tong, Kaufmann, Balesni, Stickland, Korbak, and Evans]{berglund2023reversal}
Lukas Berglund, Meg Tong, Max Kaufmann, Mikita Balesni, Asa~Cooper Stickland, Tomasz Korbak, and Owain Evans.
\newblock The reversal curse: Llms trained on" a is b" fail to learn" b is a".
\newblock \emph{arXiv preprint arXiv:2309.12288}, 2023.

\bibitem[Bronstein et~al.(2021)Bronstein, Bruna, Cohen, and Veli{\v{c}}kovi{\'c}]{bronstein2021geometric}
Michael~M Bronstein, Joan Bruna, Taco Cohen, and Petar Veli{\v{c}}kovi{\'c}.
\newblock Geometric deep learning: Grids, groups, graphs, geodesics, and gauges.
\newblock \emph{arXiv preprint arXiv:2104.13478}, 2021.

\bibitem[Conrad(2010)]{conrad2010characters}
Keith Conrad.
\newblock Characters of finite abelian groups.
\newblock \emph{Lecture Notes}, 17, 2010.

\bibitem[Diaconis(1988)]{diaconis1988group}
Persi Diaconis.
\newblock Group representations in probability and statistics.
\newblock \emph{Lecture notes-monograph series}, 11:\penalty0 i--192, 1988.

\bibitem[Du \& Lee(2018)Du and Lee]{du2018power}
Simon Du and Jason Lee.
\newblock On the power of over-parametrization in neural networks with quadratic activation.
\newblock In \emph{International conference on machine learning}, pp.\  1329--1338. PMLR, 2018.

\bibitem[Dubey et~al.(2024)Dubey, Jauhri, Pandey, Kadian, Al-Dahle, Letman, Mathur, Schelten, and et. al]{dubey2024llama3herdmodels}
Abhimanyu Dubey, Abhinav Jauhri, Abhinav Pandey, Abhishek Kadian, Ahmad Al-Dahle, Aiesha Letman, Akhil Mathur, Alan Schelten, and et. al.
\newblock The llama 3 herd of models, 2024.
\newblock URL \url{https://arxiv.org/abs/2407.21783}.

\bibitem[Dziri et~al.(2023)Dziri, Lu, Sclar, Li, Jiang, Lin, West, Bhagavatula, Le~Bras, Hwang, et~al.]{dziri2023faith}
Nouha Dziri, Ximing Lu, Melanie Sclar, Xiang~Lorraine Li, Liwei Jiang, Bill~Yuchen Lin, Peter West, Chandra Bhagavatula, Ronan Le~Bras, Jena~D Hwang, et~al.
\newblock Faith and fate: Limits of transformers on compositionality (2023).
\newblock \emph{arXiv preprint arXiv:2305.18654}, 2023.

\bibitem[Fulton \& Harris(2013)Fulton and Harris]{fulton2013representation}
William Fulton and Joe Harris.
\newblock \emph{Representation theory: a first course}, volume 129.
\newblock Springer Science \& Business Media, 2013.

\bibitem[Garrido et~al.(2024)Garrido, Assran, Ballas, Bardes, Najman, and LeCun]{garrido2024learning}
Quentin Garrido, Mahmoud Assran, Nicolas Ballas, Adrien Bardes, Laurent Najman, and Yann LeCun.
\newblock Learning and leveraging world models in visual representation learning.
\newblock \emph{arXiv preprint arXiv:2403.00504}, 2024.

\bibitem[Gromov(2023)]{gromov2023grokking}
Andrey Gromov.
\newblock Grokking modular arithmetic.
\newblock \emph{arXiv preprint arXiv:2301.02679}, 2023.

\bibitem[Huang et~al.(2023)Huang, Chen, Mishra, Zheng, Yu, Song, and Zhou]{huang2023large}
Jie Huang, Xinyun Chen, Swaroop Mishra, Huaixiu~Steven Zheng, Adams~Wei Yu, Xinying Song, and Denny Zhou.
\newblock Large language models cannot self-correct reasoning yet.
\newblock \emph{arXiv preprint arXiv:2310.01798}, 2023.

\bibitem[Huang et~al.(2024)Huang, Hu, Han, Liu, and Sun]{huang2024unified}
Yufei Huang, Shengding Hu, Xu~Han, Zhiyuan Liu, and Maosong Sun.
\newblock Unified view of grokking, double descent and emergent abilities: A perspective from circuits competition.
\newblock \emph{arXiv preprint arXiv:2402.15175}, 2024.

\bibitem[Jiang et~al.(2023)Jiang, Sablayrolles, Mensch, Bamford, Chaplot, de~las Casas, Bressand, Lengyel, Lample, Saulnier, Lavaud, Lachaux, Stock, Scao, Lavril, Wang, Lacroix, and Sayed]{jiang2023mistral7b}
Albert~Q. Jiang, Alexandre Sablayrolles, Arthur Mensch, Chris Bamford, Devendra~Singh Chaplot, Diego de~las Casas, Florian Bressand, Gianna Lengyel, Guillaume Lample, Lucile Saulnier, Lélio~Renard Lavaud, Marie-Anne Lachaux, Pierre Stock, Teven~Le Scao, Thibaut Lavril, Thomas Wang, Timothée Lacroix, and William~El Sayed.
\newblock Mistral 7b, 2023.
\newblock URL \url{https://arxiv.org/abs/2310.06825}.

\bibitem[Jin \& Rinard(2024)Jin and Rinard]{jin2024emergentrepresentationsprogramsemantics}
Charles Jin and Martin Rinard.
\newblock Emergent representations of program semantics in language models trained on programs, 2024.

\bibitem[Kambhampati et~al.(2024)Kambhampati, Valmeekam, Guan, Verma, Stechly, Bhambri, Saldyt, and Murthy]{kambhampati2024llmscantplanhelp}
Subbarao Kambhampati, Karthik Valmeekam, Lin Guan, Mudit Verma, Kaya Stechly, Siddhant Bhambri, Lucas Saldyt, and Anil Murthy.
\newblock Llms can't plan, but can help planning in llm-modulo frameworks, 2024.
\newblock URL \url{https://arxiv.org/abs/2402.01817}.

\bibitem[Khatri \& Rao(1968)Khatri and Rao]{khatri1968solutions}
CG~Khatri and C~Radhakrishna Rao.
\newblock Solutions to some functional equations and their applications to characterization of probability distributions.
\newblock \emph{Sankhy{\=a}: the Indian journal of statistics, series A}, pp.\  167--180, 1968.

\bibitem[Li et~al.(2024)Li, Liu, Zhou, and Ma]{li2024chain}
Zhiyuan Li, Hong Liu, Denny Zhou, and Tengyu Ma.
\newblock Chain of thought empowers transformers to solve inherently serial problems.
\newblock \emph{ICLR}, 2024.

\bibitem[Liu et~al.(2022)Liu, Ash, Goel, Krishnamurthy, and Zhang]{liu2022transformers}
Bingbin Liu, Jordan~T Ash, Surbhi Goel, Akshay Krishnamurthy, and Cyril Zhang.
\newblock Transformers learn shortcuts to automata.
\newblock \emph{arXiv preprint arXiv:2210.10749}, 2022.

\bibitem[Morwani et~al.(2023)Morwani, Edelman, Oncescu, Zhao, and Kakade]{morwani2023feature}
Depen Morwani, Benjamin~L Edelman, Costin-Andrei Oncescu, Rosie Zhao, and Sham Kakade.
\newblock Feature emergence via margin maximization: case studies in algebraic tasks.
\newblock \emph{arXiv preprint arXiv:2311.07568}, 2023.

\bibitem[Nanda et~al.(2023)Nanda, Chan, Lieberum, Smith, and Steinhardt]{nanda2023progress}
Neel Nanda, Lawrence Chan, Tom Lieberum, Jess Smith, and Jacob Steinhardt.
\newblock Progress measures for grokking via mechanistic interpretability.
\newblock In \emph{The Eleventh International Conference on Learning Representations}, 2023.
\newblock URL \url{https://openreview.net/forum?id=9XFSbDPmdW}.

\bibitem[Nezhurina et~al.(2024)Nezhurina, Cipolina-Kun, Cherti, and Jitsev]{nezhurina2024alicewonderlandsimpletasks}
Marianna Nezhurina, Lucia Cipolina-Kun, Mehdi Cherti, and Jenia Jitsev.
\newblock Alice in wonderland: Simple tasks showing complete reasoning breakdown in state-of-the-art large language models, 2024.
\newblock URL \url{https://arxiv.org/abs/2406.02061}.

\bibitem[OpenAI(2024)]{openai2024gpt4technicalreport}
OpenAI.
\newblock Gpt-4 technical report, 2024.
\newblock URL \url{https://arxiv.org/abs/2303.08774}.

\bibitem[Ouellette et~al.(2023)Ouellette, Pfister, and Jud]{ouellette2023counting}
Simon Ouellette, Rolf Pfister, and Hansueli Jud.
\newblock Counting and algorithmic generalization with transformers.
\newblock \emph{arXiv preprint arXiv:2310.08661}, 2023.

\bibitem[Power et~al.(2022)Power, Burda, Edwards, Babuschkin, and Misra]{power2022grokking}
Alethea Power, Yuri Burda, Harri Edwards, Igor Babuschkin, and Vedant Misra.
\newblock Grokking: Generalization beyond overfitting on small algorithmic datasets.
\newblock \emph{arXiv preprint arXiv:2201.02177}, 2022.

\bibitem[Shazeer(2020)]{shazeer2020glu}
Noam Shazeer.
\newblock Glu variants improve transformer.
\newblock \emph{arXiv preprint arXiv:2002.05202}, 2020.

\bibitem[So et~al.(2021)So, Manke, Liu, Dai, Shazeer, and Le]{DBLP:journals/corr/abs-2109-08668}
David~R. So, Wojciech Manke, Hanxiao Liu, Zihang Dai, Noam Shazeer, and Quoc~V. Le.
\newblock Primer: Searching for efficient transformers for language modeling.
\newblock \emph{NeurIPS}, 2021.
\newblock URL \url{https://arxiv.org/abs/2109.08668}.

\bibitem[Steinberg(2009)]{steinberg2009representation}
Benjamin Steinberg.
\newblock Representation theory of finite groups.
\newblock \emph{Carleton University}, 2009.

\bibitem[Sutton(2018)]{sutton2018reinforcement}
Richard~S Sutton.
\newblock Reinforcement learning: An introduction.
\newblock \emph{A Bradford Book}, 2018.

\bibitem[Team(2024{\natexlab{a}})]{deepseekai2024deepseekv2strongeconomicalefficient}
DeepSeek Team.
\newblock Deepseek-v2: A strong, economical, and efficient mixture-of-experts language model, 2024{\natexlab{a}}.
\newblock URL \url{https://arxiv.org/abs/2405.04434}.

\bibitem[Team(2024{\natexlab{b}})]{geminiteam2024gemini15unlockingmultimodal}
Gemini Team.
\newblock Gemini 1.5: Unlocking multimodal understanding across millions of tokens of context, 2024{\natexlab{b}}.
\newblock URL \url{https://arxiv.org/abs/2403.05530}.

\bibitem[Varma et~al.(2023)Varma, Shah, Kenton, Kram{\'a}r, and Kumar]{varma2023explaining}
Vikrant Varma, Rohin Shah, Zachary Kenton, J{\'a}nos Kram{\'a}r, and Ramana Kumar.
\newblock Explaining grokking through circuit efficiency.
\newblock \emph{arXiv preprint arXiv:2309.02390}, 2023.

\bibitem[Wei et~al.(2022)Wei, Tay, Bommasani, Raffel, Zoph, Borgeaud, Yogatama, Bosma, Zhou, Metzler, et~al.]{wei2022emergent}
Jason Wei, Yi~Tay, Rishi Bommasani, Colin Raffel, Barret Zoph, Sebastian Borgeaud, Dani Yogatama, Maarten Bosma, Denny Zhou, Donald Metzler, et~al.
\newblock Emergent abilities of large language models.
\newblock \emph{TMLR}, 2022.

\bibitem[Wijmans et~al.(2023)Wijmans, Savva, Essa, Lee, Morcos, and Batra]{wijmans2023emergence}
Erik Wijmans, Manolis Savva, Irfan Essa, Stefan Lee, Ari~S Morcos, and Dhruv Batra.
\newblock Emergence of maps in the memories of blind navigation agents.
\newblock \emph{AI Matters}, 9\penalty0 (2):\penalty0 8--14, 2023.

\bibitem[Wu et~al.(2019)Wu, Wang, and Liu]{wu2019splitting}
Lemeng Wu, Dilin Wang, and Qiang Liu.
\newblock Splitting steepest descent for growing neural architectures.
\newblock \emph{Advances in neural information processing systems}, 32, 2019.

\bibitem[Xie et~al.(2024)Xie, Zhang, Chen, Zhu, Lou, Tian, Xiao, and Su]{xie2024travelplannerbenchmarkrealworldplanning}
Jian Xie, Kai Zhang, Jiangjie Chen, Tinghui Zhu, Renze Lou, Yuandong Tian, Yanghua Xiao, and Yu~Su.
\newblock Travelplanner: A benchmark for real-world planning with language agents, 2024.

\bibitem[Ye et~al.(2024)Ye, Xu, Li, and Allen-Zhu]{ye2024physics}
Tian Ye, Zicheng Xu, Yuanzhi Li, and Zeyuan Allen-Zhu.
\newblock Physics of language models: Part 2.1, grade-school math and the hidden reasoning process.
\newblock \emph{arXiv preprint arXiv:2407.20311}, 2024.

\bibitem[Yehudai et~al.(2024)Yehudai, Kaplan, Ghandeharioun, Geva, and Globerson]{yehudai2024can}
Gilad Yehudai, Haim Kaplan, Asma Ghandeharioun, Mor Geva, and Amir Globerson.
\newblock When can transformers count to n?
\newblock \emph{arXiv preprint arXiv:2407.15160}, 2024.

\bibitem[Yoon et~al.(2017)Yoon, Yang, Lee, and Hwang]{yoon2017lifelong}
Jaehong Yoon, Eunho Yang, Jeongtae Lee, and Sung~Ju Hwang.
\newblock Lifelong learning with dynamically expandable networks.
\newblock \emph{arXiv preprint arXiv:1708.01547}, 2017.

\bibitem[Yu et~al.(2018)Yu, Yang, Xu, Yang, and Huang]{yu2018slimmable}
Jiahui Yu, Linjie Yang, Ning Xu, Jianchao Yang, and Thomas Huang.
\newblock Slimmable neural networks.
\newblock \emph{arXiv preprint arXiv:1812.08928}, 2018.

\bibitem[Zhang et~al.(2024)Zhang, Song, Yu, Han, Lin, Xiao, Song, Liu, Mi, and Sun]{zhang2024relu2winsdiscoveringefficient}
Zhengyan Zhang, Yixin Song, Guanghui Yu, Xu~Han, Yankai Lin, Chaojun Xiao, Chenyang Song, Zhiyuan Liu, Zeyu Mi, and Maosong Sun.
\newblock Relu$^2$ wins: Discovering efficient activation functions for sparse llms, 2024.
\newblock URL \url{https://arxiv.org/abs/2402.03804}.

\bibitem[Zhong et~al.(2024)Zhong, Liu, Tegmark, and Andreas]{zhong2024clock}
Ziqian Zhong, Ziming Liu, Max Tegmark, and Jacob Andreas.
\newblock The clock and the pizza: Two stories in mechanistic explanation of neural networks.
\newblock \emph{Advances in Neural Information Processing Systems}, 36, 2024.

\bibitem[Zhou et~al.(2024)Zhou, Fu, Sharan, and Jia]{zhou2024pre}
Tianyi Zhou, Deqing Fu, Vatsal Sharan, and Robin Jia.
\newblock Pre-trained large language models use fourier features to compute addition.
\newblock \emph{arXiv preprint arXiv:2406.03445}, 2024.

\end{thebibliography}
\bibliographystyle{iclr25_conference}

\clearpage

\appendix

\onecolumn

\section{Notation Table}
\begin{table}[h]
    \centering
    \begin{tabular}{l|l}
      Symbol  & Description \\
      \hline\hline
      $\cc$ & The set of complex numbers. The complex field. \\
      $\nn$ & The set of natural numbers. \\ 
      $\i$ & The imaginary unit. $\i = \sqrt{-1}$. \\
      $\bar a$, $\bar \va$, $\bar A$ & The complex conjugate of a scalar $a$, a vector $\va$ or a matrix $A$. \\ 
      $A^*$ & The conjugate transpose of matrix $A$. $A^* = \bar A^\top$. \\ 
      $\mathbb{I}(x)$ & The indicator function. $\mathbb{I}(x) = 1$ if $x$ is true, otherwise $0$. \\ 
      \hline\hline
      $G$   & The Abelian group to be studied. \\
      $g\in G$ & Group element $g$ in $G$. \\
      $g_1 g_2$ & Production of group element $g_1$ and $g_2$ under group multiplication. \\
      $d$   & Size of the group $G$. $|G| = d$. \\ 
      $\ve_g$ & One-hot representation of group element $g$. The dimension of $\ve_g$ is $d$. \\
      $\phi_k : G\mapsto \cc$ & The $k$-th character function of $G$. If $G$ is cyclic and $0 \le g < d$, then $\phi_k(g) = e^{\i 2\pi k g / d}$. \\
      $\vphi_k \in \cc^d$ & The $k$-th character function in vector form. $\vphi_k = [\phi_k(g)]_{g\in G}$. \\ 
      $P^\perp_1$ & Zero-mean projection matrix $P^\perp_1 \equiv I - \frac{1}{d}\vone\vone^\top$. \\
      \hline\hline
      $\vw_{aj}$, $\vw_{bj}$ & Fan-in weight vectors for node $j$ in 2-layer networks defined in Eqn.~\ref{eq:arch}. \\
      $\vw_{cj}$ & The fan-out weight vector for node $j$ in 2-layer networks defined in Eqn.~\ref{eq:arch}. \\
      $\cZ_q$ & Collection of weight (in Fourier space) of all 2-layer networks with $q$ hidden nodes. \\
      $\cZ$ & Collection of weight (in Fourier space) of all 2-layer networks. $\cZ = \bigcup_{q\ge 0} \cZ_q$. \\
      $\vz \in \cZ$ & Weight matrices of one specific instance of 2-layer network. \\ 
      $\ord(\vz)$ & The number of hidden nodes in $\vz$. \\
      $\vz_1 + \vz_2$, $\vz_1*\vz_2$ & The ring addition and multiplication (Def.~\ref{def:operationsinz}). \\
      \hline\hline
      $r : \cZ \mapsto \cc$ & The sum potential (Def.~\ref{def:sum-potential}). \\
      $R$ & Collection of sum potentials. E.g., $R_\g = \{r_{kkk}, k\neq 0\}$. \\ 
      $\con_\g$, $\con_\c$, $\con_\n$, $\con_*$ & Collections of sum potentials (Lemma~\ref{co:globalminimizer}) that appear in MSE loss function (Eqn.~\ref{eq:obj}).   
    \end{tabular}
    \caption{The notation table.}
    \label{tab:notation-table}
\end{table}

\section{Decoupling $L_2$ Loss (Proof)}
\label{sec:appendix-decoupling}
We use the \emph{character function} $\phi: G\rightarrow \cc$, which maps a group element $g$ into a complex number.  
\begin{lemma} 
For finite Abelian group, the character function $\phi$ has the following properties~\cite{fulton2013representation,steinberg2009representation}: 
\begin{itemize}
    \item It is a 1-dimensional (irreducible) representation of the group $G$, i.e., $|\phi(g)| = 1$ for $g\in G$ and for any $g_1,g_2\in G$, $\phi(g_1g_2) = \phi(g_1)\phi(g_2)$. 
    \item There exists $d$ character functions $\{\phi_k\}$ that satisfy the orthonormal condition $\frac{1}{d}\sum_{g\in G} \phi_k(g) \overline{\phi_{k'}}(g) = \mathbb{I}(k=k')$. Here $\overline{\phi}$ is the complex conjugate of $\phi$ and is also a character function.
    \item The set of character functions $\{\phi_k\}$ forms a \emph{character group} $\hat G$ under pairwise multiplication: $\phi_{k_1+k_2} = \phi_{k_1} \circ \phi_{k_2}$. 
\end{itemize}
\end{lemma}

Note that the \emph{frequency} $k$ goes from $0$ to $d-1$, where $\phi_0 \equiv 1$ is the trivial representation (i.e., all $g \in G$ maps to $1$). According to the Fundamental Theorem of Finite Abelian Groups, each finite Abelian group can be decomposed into a direct sum of cyclic groups, and the character function of each cyclic group is exactly (scaled) Fourier bases. Therefore, in Abelian group, $k$ is a multi-dimensional frequency index. ~\cite{conrad2010characters} shows that $\hat G \cong G$ (Theorem 3.13) so each character function $\phi \in\hat G$ can also be indexed by $g$ itself. Right now we keep the index $k$.

For convenience, we define $\phi_{-k} := \overline \phi_k$ as the (complex) conjugate representation of $\phi_k$. 

Let $\vphi_k = [\phi_k(g)]_{g\in G} \in \cc^{d}$ be the vector that contains the value of the character function $\phi_k$ over $G$. Then $\{\vphi_k\}$ form an orthogonal base in $\cc^{d}$ and we can represent the weight vector $\vw_{pj}$ as the following, where $p \in \{a,b,c\}$: 
\begin{equation}
    \rebut{
    \vw_{aj} = \sum_{k\neq 0} z_{akj} \vphi_k,\quad\quad \vw_{bj} = \sum_{k\neq 0} z_{bkj} \vphi_k, \quad\quad \vw_{cj} = \sum_{k\neq 0} z_{ckj} \bar\vphi_k 
    }
    \label{eq:w-freq-space} 
\end{equation}
where $\vz := \{z_{pkj}\}$ are the complex coefficients. Here $p \in \{a,b,c\}$, $0\le k < d$ and $j$ runs through hidden nodes. 

\analyticform*
\begin{proof}
Note that the objective $\ell$ can be written down as 
\begin{eqnarray}
    \ell &=& \eee{g_1,g_2}{\|P_1^\perp(\vo(g_1,g_2)/2d - \ve_{g_1g_2})\|^2} \\
    &=& \eee{g_1,g_2}{\vo^\top P_1^\perp \vo/4d^2 - \vo^\top P_1^\perp \ve_{g_1g_2}/d + \ve_{g_1g_2}^\top P_1^\perp\ve_{g_1g_2}}
\end{eqnarray}
For notation brevity, let $z_{akj} := a_{kj}$, $z_{bkj} := b_{kj}$ and $z_{ckj} := c_{kj}$. For $\ee{\vo^\top P_1^\perp \ve_{g_1g_2}}$, since 
\begin{eqnarray}
\ve_{g_1g_2}^\top P_1^\perp \vo &=&  \sum_j \ve_{g_1g_2}^\top P_1^\perp \vw_{cj} \sigma(\vw_{aj}^\top \ve_{g_1} + \vw_{bj}^\top \ve_{g_2}) \\
&=& \sum_j \left(\sum_{k'\neq 0} c_{k'j}\bar\phi_{k'}(g_1 g_2)\right) \left(\vw_{aj}^\top \ve_{g_1} + \vw_{bj}^\top \ve_{g_2}\right)^2 \\
&=& \sum_j \left(\sum_{k'\neq 0} c_{k'j}\bar\phi_{k'}(g_1 g_2)\right) \left(\sum_k \sum_{p\in \{a,b\}} z_{pkj}\phi_k(g_p)\right)^2
\end{eqnarray}
Therefore, leveraging the fact that $\bar\phi_{k'}(g_1 g_2) = \bar\phi_{k'}(g_1)\bar\phi_{k'}(g_2)$, we have:
\begin{equation}
   \eee{g_1,g_2}{\ve_{g_1g_2}^\top P_1^\perp \vo} =  
   \sum_{k_1,k_2,k'\neq 0,p_1,p_2,j} c_{k'j} z_{p_1k_1j} z_{p_2k_2j} \eee{g_1,g_2}{\bar\phi_{k'}(g_1) \bar\phi_{k'}(g_2) \phi_{k_1}(g_{p_1}) \phi_{k_2}(g_{p_2})} 
\end{equation}
Since $\eee{g}{\phi_{k}(g)\bar\phi_{k'}(g)} = \mathbb{I}(k=k')$, there are only a few cases that the summand is nonzero: 
\begin{itemize}
\item $p_1 = a$, $p_2 = b$, $k' = k_1 = k_2 \neq 0$. 
\item $p_1 = b$, $p_2 = a$, $k' = k_1 = k_2 \neq 0$. 
\end{itemize}
In both cases, the summation reduces to $\sum_{k\neq 0,j} c_{kj}z_{akj}z_{bkj} = \sum_{k\neq 0,j} c_{kj}a_{kj}b_{kj}$. Let $r_{k_1k_2k'} := \sum_{j} a_{k_1j} b_{k_2j}c_{k'j}$, then we have 
\begin{equation}
    \eee{g_1,g_2}{\vo^\top(g_1,g_2) P_1^\perp \ve_{g_1g_2}} = 2\sum_{k\neq 0,j} a_{kj}b_{kj}c_{kj} = 2\sum_{k\neq 0} r_{kkk}
\end{equation}

For $\ee{\vo^\top P_1^\perp \vo}$, we have:
\begin{equation}
    \vo^\top P_1^\perp \vo = \sum_{j,j'} \vw_{cj}^\top P_1^\perp \vw_{cj'} \sigma(\vw_{aj}^\top\ve_{g_1} + \vw_{bj}^\top\ve_{g_2}) \sigma(\vw_{aj'}^\top\ve_{g_1} + \vw_{bj'}^\top\ve_{g_2}) 
\end{equation}
here
\begin{equation}
 \vw_{cj}^\top P_1^\perp \vw_{cj'} = \left(\sum_{k'\neq 0} c_{k'j}\bar\vphi_{k'}\right)^\top \left(\sum_{k''\neq 0}\bar c_{k''j'}\vphi_{k''}\right) = d\sum_{k'\neq 0} c_{k'j} \bar c_{k'j'}
\end{equation}
due to the fact that $\bar\vphi^\top_k\vphi_{k'} = \sum_g \bar\phi_k(g)\phi_{k'}(g) = d\mathbb{I}(k=k')$.

Then the key part is to compute the following terms:
\begin{equation}
    \eee{g_1,g_2}{
    z_{p_1k_1j_1} z_{p_2k_2j_1} z_{p_3k_3j_2} z_{p_4k_4j_2} c_{k'j_1}\bar c_{k'j_2} \phi_{k_1}(g_{p_1})\phi_{k_2}(g_{p_2}) \phi_{k_3}(g_{p_3})\phi_{k_4}(g_{p_3})} \label{eq:long-sum}
\end{equation}
summing over $\{p_1,p_2,p_3,p_4,k_1,k_2,k_3,k_4,k'\neq 0,j_1,j_2\}$. Note that since each $p \in \{a,b\}$, there are $2^4=16$ choices of $(p_1, p_2, p_3, p_4)$. For notation brevity, we use $(1, 3)$ to represent the subset of $p$ that takes the value of $a$ (e.g., $(1, 3)$ means that $p_1=p_3 = a$ and $p_2=p_4=b$). It is clear that for odd assignments such as $(1,2,3)$, since $z_{p0j} = 0$, the summation is zero. Then, we only discuss the even cases as follows:

\textbf{Case 1: $(1,3)$, $(2,4)$, $(1, 4)$, $(2,3)$}. The 4 cases are identical so we only need to analyze one. We take $(1,3)$ as an example. For $(1,3)$, $p_1 = p_3 = a$, $p_2 = p_4 = b$ and the only nonzero terms is when $k_1 + k_3 = 0 \mod d$, $k_2 + k_4 = 0 \mod d$, since $\eee{g_1}{\phi_{k_1}(g_1)\phi_{k_3}(g_1)} = \mathbb{I}(k_1+k_3=0 \mod d)$ (and similar in other cases). Then Eqn.~\ref{eq:long-sum} becomes:
\begin{eqnarray}
    & & \sum_{k_1,k_2,k'\neq 0} \sum_{j_1j_2}
    z_{a k_1j_1} z_{b k_2j_1} z_{a,-k_1,j_2} z_{b, -k_2, j_2} c_{k'j_1}\bar c_{k'j_2} \\
    &=& \sum_{k_1,k_2,k'\neq 0} \sum_{j_1} z_{a k_1j_1} z_{b k_2j_1}c_{k'j_1} \overline{\sum_{j_2}z_{a k_1j_2} z_{b k_2j_2}c_{k'j_2}} \\
    &=& \sum_{k_1,k_2,k'\neq 0} \sum_{j_1} a_{k_1j_1} b_{k_2j_1}c_{k'j_1} \overline{\sum_{j_2}a_{k_1j_2} b_{k_2j_2}c_{k'j_2}} \\
    &=& \sum_{k_1,k_2,k'\neq 0} r_{k_1k_2k'} \overline{r_{k_1k_2k'}} = \sum_{k_1,k_2,k'\neq 0} |r_{k_1k_2k'}|^2
\end{eqnarray}
Since there are 4 such cases, we have:
\begin{equation}
    \epsilon_1 = 4\sum_{k'\neq 0} \sum_{k_1k_2} |r_{k_1k_2k'}|^2
\end{equation}

\textbf{Case 2: $(1,2)$ and $(3,4)$}. The two cases are identical. Take $(1,2)$ as an example. In this case, $p_1 = p_2 = a$ and $p_3=p_4=b$. The only non-zero terms are when $k_1+k_2 = 0$, $k_3+k_4=0$. Then Eqn.~\ref{eq:long-sum} becomes:
\begin{eqnarray}
    & & \sum_{k_1,k_3,k'\neq 0} \sum_{j_1j_2}
    z_{a k_1j_1} \bar z_{ak_1j_1} z_{bk_3j_2} \bar z_{bk_3j_2} c_{k'j_1}\bar c_{k'j_2} \\
    &=& \sum_{k_1,k_3,k'\neq 0} \sum_{j_1} |a_{k_1j_1}|^2 c_{k'j_1} \sum_{j_2} |b_{k_3j_2}|^2 \bar c_{k'j_2} \\
    &=& \sum_{k'\neq 0} \left[\sum_{j_1} \left(\sum_{k_1} |a_{k_1j_1}|^2\right) c_{k'j_1} \right] \left[\sum_{j_2} \left(\sum_{k_3}|b_{k_3j_2}|^2\right) \bar c_{k'j_2}\right]
\end{eqnarray}
Let $r^\circledast_{amk'} := \sum_j \left(\sum_{k_1+k_2=m} a_{k_1j} a_{k_2j}\right) c_{k'j}$ (similar for $r^\circledast_{bmk'}$), then the above becomes $\sum_{k'\neq 0} r^{\circledast}_{a0k'} \bar r^{\circledast }_{b0k'}$. 

Similarly, for $(3,4)$, the above equation becomes $\sum_{k'\neq 0} \bar r^{\circledast}_{a0k'} r^{\circledast}_{b0k'}$. Therefore, we have:
\begin{equation}
    \epsilon_2 = \sum_{k'\neq 0} r^{\circledast}_{a0k'} \bar r^{\circledast}_{b0k'} + \bar r^{\circledast}_{a0k'} r^{\circledast}_{b0k'}
\end{equation}

Note that this term can be negative. However, we will see that when it is combined with the following terms, all terms will be non-negative.  

\textbf{Case 3: $(1,2,3,4)$ and $()$}. In this case we have:
\begin{eqnarray}
    & & \sum_{k'\neq 0} \sum_{j_1j_2}\sum_{p\in \{a,b\}} \sum_{k_1+k_2+k_3+k_4=0} z_{pk_1j_1} z_{pk_2j_1} z_{pk_3j_2} z_{pk_4j_2} c_{k'j_1} \bar c_{k'j_2} \\
    &=& \sum_{k'\neq 0} \sum_{j_1j_2}\sum_{p\in \{a,b\}} \sum_{k_1+k_2=k_3+k_4} z_{pk_1j_1} z_{pk_2j_1} \bar z_{pk_3j_2} \bar z_{pk_4j_2} c_{k'j_1} \bar c_{k'j_2} \\
    &=& \sum_{k'\neq 0} \sum_m \sum_{p\in \{a,b\}} \sum_{j_1j_2}\sum_{k_1+k_2=m}\sum_{k_3+k_4=m} z_{pk_1j_1} z_{pk_2j_1} \bar z_{pk_3j_2} \bar z_{pk_4j_2} c_{k'j_1} \bar c_{k'j_2} \\
    &=& \sum_{k'\neq 0} \sum_m \sum_{p\in \{a,b\}} \left[\sum_{j_1} \left(\sum_{k_1+k_2=m} z_{pk_1j_1} z_{pk_2j_1}\right) c_{k'j_1}\right]\left[ \sum_{j_2}\left(\sum_{k_3+k_4=m} \overline{z_{pk_3j_2} z_{pk_4j_2}}\right) \bar c_{k'j_2}\right] \nonumber \\ 
    &=& \sum_{k'\neq 0} \sum_m |r^{\circledast}_{amk'}|^2 + |r^{\circledast}_{bmk'}|^2
\end{eqnarray}
In particular, when $m = 0$, we have $\sum_{k'\neq 0} |r^{\circledast}_{a0k'}|^2 + |r^{\circledast}_{b0k'}|^2$. Therefore, we have
\begin{equation}
    \epsilon_2 + \epsilon_{3, m=0} = \sum_{k'\neq 0} |r^{\circledast}_{a0k'} + r^{\circledast}_{b0k'}|^2
\end{equation}
Finally, putting them together, we have:
\begin{eqnarray}
    \ee{\vo^\top P_1^\perp \vo} &=& d(\epsilon_1 + \epsilon_2 + \epsilon_3) = d(\epsilon_1 + \left(\epsilon_2 + \epsilon_{3,m=0}\right) + \epsilon_{3,m\neq 0}) \\
    &=& d\sum_{k'\neq 0} \left(4\sum_{k_1k_2} |r_{k_1k_2k'}|^2 + |r^{\circledast}_{a0k'} + r^{\circledast}_{b0k'}|^2 + \sum_{m\neq 0} |r^{\circledast}_{amk'}|^2 + |r^{\circledast}_{bmk'}|^2\right) \nonumber \\
    &\ge & 0 
\end{eqnarray}
Putting them together, we arrived at the conclusion.
\end{proof}

\globalsolutions*
\begin{proof}
Note that $2\sum_k r_{kkk} - \sum_k |r_{kkk}|^2$ has a minimizer $r_{kkk} = 1$. Therefore, the best loss value any assignment of weights is able to achieve is the following:
\begin{align}
    r_{k_1k_2k'} & = \sum_j a_{k_1j}b_{k_2j}c_{k'j} = \mathbb{I}(k_1=k_2=k') & k' \neq 0\\
    r^{\circledast}_{a0k'} + r^{\circledast}_{b0k'} & := \sum_j \left(\sum_k |a_{kj}|^2 + |b_{kj}|^2\right) c_{k'j} = 0 & k' \neq 0 \\
    r^{\circledast}_{amk'} & := \sum_j \left(\sum_{k_1+k_2=m} a_{k_1j}a_{k_2j}\right) c_{k'j} = 0 & k' \neq 0, m \neq 0 \\
    r^{\circledast}_{bmk'} & := \sum_j \left(\sum_{k_1+k_2=m} b_{k_1j}b_{k_2j}\right) c_{k'j} = 0 & k' \neq 0, m \neq 0 
\end{align}
Therefore the sufficient conditions (Eqn.~\ref{eq:globalminimizer}) will make all above come true. 
\end{proof}

\section{Semi-ring structure of $\cZ$ (Proof)}
\zring*
\begin{proof}
Straightforward from the definition of addition and multiplication (Def.~\ref{def:operationsinz}) and identification of hidden nodes under permutation (Def.~\ref{def:permutation-invariant}). Note that ring addition (i.e., concatenation) does not have inverse and thus it is a semi-ring.  
\end{proof}

\def\supp{\mathrm{supp}}

\pothomo*
\begin{proof}
Let $r(\vz) = \sum_j \prod_{(p, k) \in \idx(r)} z_{pkj}$. Since the ring identity $\vone$ is order-1 and all $z_{pkj} = 1$, it is obvious that $r(\vone) = 1$. 

Let $\supp(\vz_1)$ be the subset of the hidden nodes that corresponds to $\vz_1$ in the concatenated solution $\vz_1\ringadd\vz_2$, similar for $\supp(\vz_2)$. Note that 
\begin{equation}
    r(\vz_1+\vz_2) = \sum_{j\in \supp(\vz_1)} \prod_{(p, k) \in \idx(r)} z^\bone_{pkj} + \sum_{j\in \supp(\vz_2)} \prod_{(p, k) \in \idx(r)} z^\btwo_{pkj} = r(\vz_1) + r(\vz_2)
\end{equation}
On the other hand, we have
\begin{eqnarray}
    r(\vz_1\ringmul\vz_2) &=& \sum_{j_1j_2} \prod_{(p, k) \in \idx(r)} \left(z^\bone_{pkj_1} z^\btwo_{pkj_2}\right) \\
    &=& \sum_{j_1j_2} \left(\prod_{(p, k) \in \idx(r)} z^\bone_{pkj_1}\right)\left( 
\prod_{(p, k) \in \idx(r)} z^\btwo_{pkj_2}\right) \\
    &=& \left(\sum_{j_1} \prod_{(p, k) \in \idx(r)} z^\bone_{pkj_1}\right)\left(\sum_{j_2} \prod_{(p, k) \in \idx(r)} z^\bone_{pkj_2}\right) \\
    &=& r(\vz_1) r(\vz_2)
\end{eqnarray}
\end{proof}

\algebraglobaloptimizer*
\begin{proof}
Straightforward by leveraging the property of ring homomorphism. E.g., 
\begin{equation}
    r_{kkk}(\vz \ringmul \vy) = r_{kkk}(\vz) r_{kkk}(\vy) = r_{kkk}(\vz)
\end{equation}
and the proof is complete.
\end{proof}

\section{Solution Construction (Proof)}
\subsection{Construction of Partial Solutions}
\polyconst*
\begin{proof}
By definition, for any $r \in \con$ we have:
\begin{equation}
    r(\vz(\vu)) = \prod_{s\in \Omega_\con(\vu)} (r(\vu) + r(\hat\vs)) = \prod_{s\in \Omega_\con(\vu)} (r(\vu) - s) = 0
\end{equation}
similarly for any $r_{kkk} \in R_+$ we have:
\begin{equation}
    r_{kkk}(\vz(\vu)) = \prod_{s\in \Omega_\con(\vu)} (r_{kkk}(\vu) + r_{kkk}(\hat\vs)) = \prod_{s\in \Omega_\con(\vu)} (1 - s) \neq 0
\end{equation}
which is constant over different $k$. So $\vz(\vu)$ satisfies Lemma~\ref{co:globalminimizer}, up to a scaling factor. 
\end{proof}

\subsection{Construction of Global Solutions}
\ordersix*
\begin{proof}
Just notice that $\vz_\syn := \vrho(\vu_\syn)= \vu^2_\syn + \vu_\syn + \vone_k$ (superscript $(k)$ are omitted for brevity) makes all MPs in $R_\n$, $R_\c$ and part of $R_*$ (Tbl.~\ref{tab:poly-construction}) equal to $0$, except for ``aac'' and ``bbc'', which corresponds to monomial polynomials $r_{akkk} := \sum_j z_{akj}z_{akj}z_{ckj}$ and $r_{bkkk} := \sum_j z_{bkj}z_{bkj}z_{ckj}$. On the other hand, according to Tbl.~\ref{tab:poly-construction}, $\vz_{\nu} := \vu_{\nu} + \vone_k$ has $r_{akkk}(\vz_{\nu}) = r_{bkkk}(\vz_{\nu}) = 0$. Therefore, using ring homomorphism, we know that for any $r \in R_\n\cup R_\c \cup R_*$, $r(\vz_\syn \ringmul \vz_{\nu}) = 0$ and thus $R_\n\cup R_\c \cup R_*$ is the 0-sets. 

On the other hand for any $k'$, we have: 
\begin{align}
      r_{k'k'k'}(\vz_{F6}) &= r_{k'k'k'}\left(\frac{1}{\sqrt[3]{6}}\sum_{k=1}^{(d-1)/2} \vz^\bk_\syn \ringmul \vz^\bk_{\nu} \ringmul \vy_k\right) \\
      &= \frac{1}{6} \sum_{k=1}^{(d-1)/2} r_{k'k'k'}(\vz^\bk_\syn \ringmul \vz^\bk_{\nu} \ringmul \vy_k) \\
     &= \frac{1}{6} \sum_{k=1}^{(d-1)/2} 6 (\mathbb{I}(k=k') + \mathbb{I}(k=-k')) = 1 
\end{align}
The last equality is due to the fact that we only sum over half of the frequency. This means that $R_\g$ is a 1-set of $\vz_{F6}$. Therefore, $\vz_{F6}$ satisfies the sufficient condition (Eqn.~\ref{eq:globalminimizer}) and the conclusion follows.  
\end{proof}

\perfectmem*
\begin{proof}
Simply plugging in the solution and check whether the equations specified the equations. For $\vz_a$, for $k = 0$ everything is zero; for $k\neq 0$, we have: 
\begin{eqnarray}
    r_{k_1k_2k}(\vz_a) &=& \sum_j a_{k_1j}b_{k_2j}c_{kj} = \sum_j \omega^{j(k_1-k)} = \mathbb{I}(k_1=k\neq 0) \\
    r_{amk'k}(\vz_a) &=& \sum_j a_{k'j}a_{m-k',j}c_{kj} = \sum_j \omega^{j(m-k)} = \mathbb{I}(m=k\neq 0) \\
    r_{bmk'k}(\vz_a) &=& \sum_j b_{k'j}b_{m-k',j}c_{kj} = \sum_j \omega^{-jk} = \mathbb{I}(k=0) = 0 
\end{eqnarray}
Therefore, $\vz_a \in R_\c(k_1 \neq k) \cap R_\n \cap R_\conv(p = b \mathrm{\ or\ } m \neq k)$. Similar for $\vz_b$. For $\vz_M := d^{-2/3} \vz_a \ringmul \vz_b$, it satisfies all 0-sets constraints (i.e., for any $r$, either $\vz_a$ satisfies with $r(\vz_a) = 0$, or $\vz_b$ satisfies with $r(\vz_b) = 0$) and we have:
\begin{equation}
    r_{kkk}(d^{-2/3}\vz_a\ringmul\vz_b) = d^{-2} r_{kkk}(\vz_a)r_{kkk}(\vz_b) = d^{-2} \cdot d \cdot d = 1 
\end{equation}
So $\vz_M$ satisfies the sufficient conditions (Eqn.~\ref{eq:globalminimizer}).
\end{proof}

\orderfour*
\begin{proof}
First, $\vu_{\nu=\i} = \vu_{4\c}$ in Tbl.~\ref{tab:poly-construction} and thus $\vrho(\vu_{\nu=\i})$ has 0-sets $R_\c$ and $R_*$ except for ``$ab\bar c$'', which corresponds to MP $r_{k,k,-k} \in R_\c$. On the other hand, we have  
\begin{equation}
    r_{k,k,-k}(\vz_\xi) = 1 + \xi \cdot (-\i \bar\xi) \cdot (-\i) = 0  
\end{equation}
With the property of ring homomorphism, the conclusion follows. 
\end{proof}

\foursixsol*
\begin{proof}
While $\vz^\bk_{F4}$ does not satisfy $R_\n$, a weaker condition for a global optimizer to Theorem~\ref{thm:analyticform} is that $\sum_{k'} r_{p,k',-k',m} = 0$. We show that by adding constants to $(c,k)$ entries of $\vz^{(k_0)}_{F6}$ for $k\neq \pm k_0$, we can achieve that while not changing the value of other MPs.  

To see this, we compute for each $m \neq \pm k_0$:
\begin{align}
    & \sum_{k'} r_{p,k',-k',m}(\hat \vz^{(k_0)}_{F6}) = 2\sum_{k'}\sum_j |[\hat \vz^{(k_0)}_{F6}]_{pk'j}|^2 [\hat \vz^{(k_0)}_{F6}]_{cmj} \\
    & = 2\sum_j |[\hat \vz^{(k_0)}_{F6}]_{pk_0j}|^2 [\hat \vz^{(k_0)}_{F6}]_{cmj} = 2\sum_j [\hat \vz^{(k_0)}_{F6}]_{cmj}
\end{align}
The second equality is because all $(a,k')$ and $(b,k')$ entries are $0$ except for $k' = \pm k_0$, and the last equality is because all nonzero entries of $\vz^{(k_0)}_{F6}$ have magnitude $1$.  

On the other hand, we have:
\begin{align}
    \sum_{k'} r_{p,k',-k',m}\left(\sum_{k\neq k_0} \vz^\bk_{F4}\right) &= \sum_{k'} r_{p,k',-k',m}(\vrho(\vu_{4\c}^{(m)})) r_{p,k',-k',m}(\vz_\xi^{(m)}) \\
    &= 2 r_{p,m,-m,m}(\vrho(\vu_{4\c}^{(m)})) r_{p,m,-m,m}(\vz_\xi^{(m)}) \\
    &= 2 (1 + 1) (1+\i) = 4(1+\i) 
\end{align}

For $m= \pm k_0$, we have $r_{p,k',-k',m}(\hat \vz^{(k_0)}_{F6}) = 0$ and $r_{p,k',-k',m}(\vz^{(k)}_{F4}) = 0$ for $k\neq m$.

\begin{figure}
    \centering
    \includegraphics[width=0.8\textwidth]{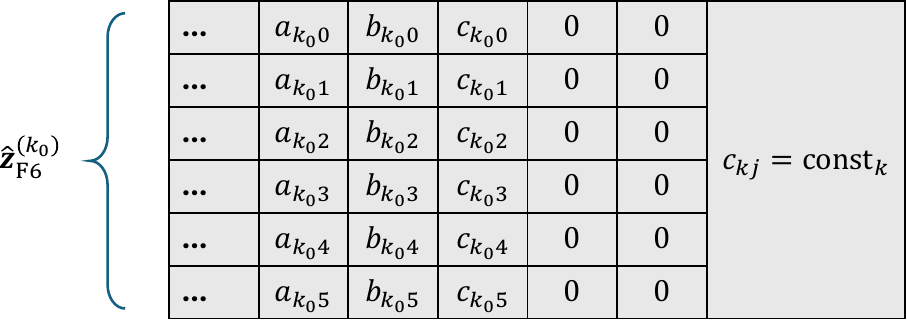}
    \caption{Visualization of $\hat \vz^{(k_0)}_{F6}$.}
    \label{fig:visualize-mixed}
\end{figure}

Therefore, we just let 
\begin{equation}
[\hat \vz^{(k_0)}_{F6}]_{cmj} = -\frac{4(1+\i)}{2\cdot 6} = -\frac13 (1 + \i) 
\end{equation}
and $\sum_{k'} r_{p,k',-k',m}(\vz_{F4/6}) = 0$ for all $m$. See Fig.~\ref{fig:visualize-mixed} for the construction. 

To see why such a modification of $\vz^{(k_0)}_{F6}$ won't change other MPs, simply notice that candidate MPs that may not be zero anymore are $r_{\pm k_0\pm k_0m}$, $r_{pk_0k_0m}$ and $r_{p,-k_0,-k_0,m}$ for $m\neq \pm k_0$. For $m=\pm k_0$, $\vz^{(k_0)}_{F6}$ are well behaved. 

Note that $r_{\pm k_0 \pm k_0k}(\hat \vz^{(k_0)}_{F6})$ is the same as applying $r_{\pm k_0 \pm k_0k_0}$ to a solution $\hat\vz$ which replaces $(c, k_0)$ entries of $\hat \vz^{(k_0)}_{F6}$ by $(c,m)$ entries. Let $\hat\vu_\syn = [\omega_3, \omega_3, 1]$ and $\hat\vu_\one = [1, -1, 1]$. Then $\hat\vz = \vrho(\hat\vu_\syn) * \vrho(\hat\vu_\one)$ and thus for $m\neq \pm k_0$, we have:
\begin{align}
        r_{\pm k_0\pm k_0m}(\vz_{F4/6}) &= r_{\pm k_0\pm k_0m}(\hat \vz^{(k_0)}_{F6}) \propto r_{\pm k_0 \pm k_0k_0}(\hat\vz) \\
        &= r_{\pm k_0 \pm k_0k_0}(\vrho(\hat\vu_\syn)) r_{\pm k_0 \pm k_0k_0}(\vrho(\hat\vu_\one)) = 0
\end{align}
since $r_{\pm k_0 \pm k_0 k_0}(\vrho(\hat\vu_\one)) = 0$. Similarly for $m \neq \pm k_0$, 
\begin{align}
    r_{pk_0k_0m}(\vz_{F4/6}) &= 
    r_{pk_0k_0m}(\hat \vz^{(k_0)}_{F6}) \propto r_{pk_0k_0k_0}(\hat\vz) \\
    &= r_{pk_0k_0k_0}(\vrho(\hat\vu_\syn)) r_{pk_0k_0k_0}(\vrho(\hat\vu_\one)) = 0
\end{align}
since $r_{p k_0 k_0 k_0}(\vrho(\hat\vu_\syn)) = 0$. Similarly for $r_{p,-k_0,-k_0,m}$.
\end{proof}
\textbf{Remarks.} To construct $\hat\vz_{F6}$, in addition to $\vz_\syn\ringmul\vz_{\nu=1}$ shown in the main proof, we could use other compositions to achieve the same effects. For example, $\vz_{\syn,\alpha\beta}\ringmul\vz_{\nu=\i}$, where $\vz_{\syn,\alpha\beta}$ is:
\begin{equation}
    z_{ak\cdot}= [1,\omega_3\alpha,\bar\omega_3\beta],\quad z_{bk\cdot}=[1,\omega_3\bar\alpha, \bar\omega_3\bar\beta],\quad z_{ck\cdot}= [1, \omega_3, \bar\omega_3] \label{eq:synab}
\end{equation}
where $|\alpha| = |\beta| = 1$. Note that $\vz_{\syn} = \vrho(\vu_\syn)$ is a special case of $\vz_{\syn,\alpha\beta}$ when $\alpha=\beta=1$.

\subsection{Canonical Forms}
\label{sec:canonical-form}
\begin{definition}
\label{def:canonical}
A solution $\vz$ is called \emph{canonical at $k_0$}, or $\vz \in \cC_{k_0}$, if $z_{pk0} = 1$ for all $p$ and $k=\pm k_0$. 
\end{definition}

\begin{restatable}[Canonical Decomposition]{lemma}{candecomp}
Any solution $\vz$ with $r_{k_0k_0k_0}(\vz) \neq 0$ can be decomposed into $\vz = \vz' * \vy$, where $\vz'$ is canonical at $k_0$ and $\ord(\vy) = 1$. Both $r_{k_0k_0k_0}(\vz')\neq 0$ and $r_{k_0k_0k_0}(\vy)\neq 0$.
\end{restatable}
\begin{proof}
Since $r_{k_0k_0k_0}(\vz) = \sum_j a_{k_0j} b_{k_0j} c_{k_0j} \neq 0$, there must exist some $j$ so that $z_{ak_0j} z_{bk_0j} z_{ck_0j} \neq 0$, which means that $z_{ak_0j} \neq 0$, $z_{bk_0j} \neq 0$ and $z_{ck_0j} \neq 0$. Since the node index $j$ can be permuted, we can let node $j$ be the first node $0$ and let $y_{pk0} = z_{pkj}$ and $z'_{pkj'} = z_{pkj'} z_{pkj}^{-1}$ for $p\in \{a,b,c\}$ and $k = \pm k_0$, then $\vz'$ is canonical at $k_0$ and $\ord(\vy) = 1$. Finally, by ring homomorphism, since 
\begin{equation}
    r_{k_0k_0k_0}(\vz) = r_{k_0k_0k_0}(\vz') r_{k_0k_0k_0}(\vy) \neq 0  
\end{equation}
we know that both $r_{k_0k_0k_0}(\vz')\neq 0$ and $r_{k_0k_0k_0}(\vy)\neq 0$.
\end{proof}

\rebut{
\begin{restatable}[Necessary Condition for $R_\c$]{lemma}{satisfyra}
\label{lemma:order12-necessary}
All order-1 and order-2 solutions satisfying $R_\c := \{r_{k_1k_2k} = 0, k_1, k_2,k\mathrm{\ not\ all\ equal}\}$ must have $r_{kkk} = 0$ for all $k$ (i.e. the first equation in Eqn.~\ref{eq:globalminimizer} cannot be satisfied). 
\end{restatable}
\begin{proof}
For any order-1 solution, for any $k$, in order to make $r_{k,-k,k} = z_{ak0}z_{b,-k,0}z_{ck0} = z_{ak0}\bar z_{bk0}z_{ck0} = 0$, either $z_{ak0}$, $z_{bk0}$ or $z_{ck0}$ has to be zero, which means that $r_{kkk} = 0$. 

For order-2, first of all if any $z_{pk0} = 0$ for any $p \in \{a,b,c\}$, then a constraint like $r_{k,k,-k} = z_{ak0}z_{bk0}\bar z_{ck0} + z_{ak1}z_{bk1}\bar z_{ck1} = 0$ yields $z_{ak1}z_{bk1} z_{ck1} = 0$ and thus $r_{kkk} = 0$. If not, then for any two complex numbers $z_{pk0}$ and $z_{pk1}$, there always exist four real numbers $\theta_p \in (-\pi,\pi]$, $\theta'_p \in (-\pi,\pi]$, $m_{p0} > 0$ and $m_{p1} > 0$ so that
\begin{equation}
    z_{pk0} = m_{p0} e^{\i \theta'_p}e^{\i \theta_p}, \quad\quad z_{pk1} = m_{p1} e^{\i \theta'_p}e^{-\i \theta_p} 
\end{equation}
Then a constraint like $r_{k,k,-k} = z_{ak0}z_{bk0}\bar z_{ck0} + z_{ak1}z_{bk1}\bar z_{ck1} = 0$ can be written as $z_{ak0}z_{bk0}\bar z_{ck0} = - z_{ak1}z_{bk1}\bar z_{ck1}$, or equivalently: 
\begin{eqnarray}
    m_{a0}m_{b0}m_{c0} e^{\i (\theta'_a + \theta'_b + \theta'_c)}
    e^{\i (\theta_a + \theta_b - \theta_c)} &=& - m_{a1}m_{b1}m_{c1} e^{\i (\theta'_a + \theta'_b + \theta'_c)} e^{-\i (\theta_a + \theta_b - \theta_c)} \\
    m_{a0}m_{b0}m_{c0} e^{\i \theta_a}e^{\i \theta_b}e^{-\i \theta_c} &=& - m_{a1}m_{b1}m_{c1} e^{-\i \theta_a}e^{-\i\theta_b}e^{\i \theta_c}
\end{eqnarray}
Comparing their magnitude and phase, we have $m_{a0}m_{b0}m_{c0} = m_{a1}m_{b1}m_{c1}$ and 
\begin{equation}
    \theta_a + \theta_b - \theta_c = \pm \pi/2 \mod 2\pi
\end{equation}
Similarly, we have:
\begin{equation}
    \theta_a + \theta_c - \theta_b = \pm \pi/2 \mod 2\pi,\quad\quad
    \theta_b + \theta_c - \theta_a = \pm \pi/2 \mod 2\pi
\end{equation}
Solving the three equations and we have 6 possible solutions:
\begin{align}
   (\theta_a, \theta_b, \theta_c) &= (0, 0, \pm \pi/2)\mod 2\pi \\ 
   (\theta_a, \theta_b, \theta_c) &= (0, \pm \pi/2, 0)\mod 2\pi \\ 
   (\theta_a, \theta_b, \theta_c) &= (\pm \pi/2, 0, 0)\mod 2\pi
\end{align}
For all such solutions, let $m := m_{a0}m_{b0}m_{c0} = m_{a1}m_{b1}m_{c1}$, then we have:
\begin{eqnarray}
r_{kkk} &=& z_{ak0}z_{bk0}z_{ck0} + z_{ak1}z_{bk1}z_{ck1} \\
&=& m e^{\i(\theta'_a + \theta'_b + \theta'_c)} (e^{\i (\theta_a + \theta_b + \theta_c)} + e^{-\i (\theta_a + \theta_b + \theta_c)}) \\
&=& m e^{\i(\theta'_a + \theta'_b + \theta'_c)} (e^{\pm \i\pi/2} + e^{\mp \i\pi/2}) \\
&=& 0
\end{eqnarray}
\end{proof}
}

\rebut{

\begin{restatable}[Property of order-3 solutions satisfying $R_\c$ and $R_\g$]{lemma}{satisorderthree}
\label{lemma:order3-necessary}
With small $L_2$ regularization, all per-frequency order-3 canonical solutions $\vz$ at frequency $k_0$ that satisfy $R_\c$ and $R_\g$ are in the following form:
\begin{equation}
    z_{pk_0\cdot} = [1, \alpha_p \omega_3, \beta_p \bar\omega_3], \quad\quad \mathrm{for\ } p\in \{a,b,c\}
\end{equation}
where $\alpha_p = \pm 1$ and $\beta_p = \pm 1$ with the constraint that $\alpha_a\alpha_b\alpha_c = \beta_a\beta_b\beta_c = 1$. For $k \neq k_0, z_{pk\cdot} = 0$.
\end{restatable}
\begin{proof}
We first prove that $\vz$ satisfies $R_\c$ and $R_\g$. To see this, we have 
\begin{align}
    r_{k_1k_2k} &= \sum_j \mathbb{I}(k_1=k_2=k=k_0) \omega_3^{3j} + \sum_j \mathbb{I}(-k_1=k_2=k=k_0) \omega_3^j \\
    &+ \ldots + \sum_j \mathbb{I}(-k_1=-k_2=-k=k_0) \bar\omega_3^{3j} \\
    &= 3\mathbb{I}(k_1=k_2=k=k_0) + 3\mathbb{I}(k_1=k_2=k=-k_0) 
\end{align}
Note that all cross terms are gone since $\sum_j \omega^j_3 = 0$. It is clear that $r_{k_1k_2k} \neq 0$ unless $k_1 = k_2 = k$ so $\vz$ satisfies $R_\c$ and $R_\g$. 

Now we consider any per-frequency order-3 canonical solution (Def.~\ref{def:canonical}) at frequency $k$. Let $a_j := z_{akj}$, $b_j := z_{bkj}$ and $c_j := z_{ckj}$. Let $\va = [a_j]\in \cc^3$, $\vb = [b_j]\in \cc^3$ and $\vc = [c_j]\in \cc^3$. Since the solution is canonical, we have $a_0=b_0=c_0 = 1$. 

Then the conditions yield that
\begin{equation}
    (\va \circ \bar\vb)^\top \vc = 0,\quad 
    (\va \circ \bar\vb)^\top \bar \vc = 0,\quad
    (\bar\va \circ \vb)^\top \vc = 0,\quad 
    (\bar\va \circ \vb)^\top \bar \vc = 0
\end{equation}
which means that in $\rr^3$ space, the following condition holds: 
\begin{equation}
    \spann(\Re(\va\circ \bar\vb), \Im(\va\circ \bar\vb)) \perp \spann(\Re(\vc), \Im(\vc))
    \label{eq:orth-cond}
\end{equation}
where $\Re(\cdot)$ and $\Im(\cdot)$ are real and imaginary parts of a complex vector. Since Eqn.~\ref{eq:orth-cond} holds in $\rr^3$, it must be the following cases: either $\Re(\va\circ \bar\vb)$ is co-linear with $\Im(\va\circ \bar\vb)$, or $\Re(\vc)$ is co-linear with $\Im(\vc)$. 

If the latter is true (i.e., there exists $\beta$ so that $\beta\Re(\vc) = \Im(\vc)$), then since $c_0 = 1$ is real, $\beta = 0$ and $\Im(\vc) = 0$. So $\vc$ is real. In this case, 
\begin{equation}
    r_{kkk} = (\va \circ \vb)^\top \vc = (\va \circ \vb)^\top \bar\vc = 0   
\end{equation}
If the former is true, then similarly we conclude that $\Im(\va \circ \bar \vb) = 0$ and $\va \circ \bar \vb$ is real. Applying the same reasoning symmetrically, in order to find cases such that $r_{kkk} \neq 0$, a necessary condition is that 
\begin{equation}
    \va \circ \bar \vb, \vb \circ \bar \vc, \vc \circ \bar \va \in \rr^3
\end{equation}

Let $z_{pkj} = |z_{pkj}| e^{\i\theta_{pj}}$. Let's first consider the case that $\va \circ \bar \vb, \vb \circ \bar \vc, \vc \circ \bar \va \in \rr^3_{\ge 0}$. Then we have $\theta_{a0} = \theta_{b0} = \theta_{c0} = \theta_0 = 0$, $\theta_{a1} = \theta_{b1} = \theta_{c1} = \theta_1$, $\theta_{a2} = \theta_{b2} = \theta_{c2} = \theta_2$. Letting $m_j := |a_j| |b_j| |c_j|$, then the corresponding $r_{kkk}$ can be written as:
\begin{equation}
r_{kkk} = \sum_{j=0}^2 m_j e^{3\i\theta_j} 
\end{equation}
with the constraints that $\sum_{j=0}^2 m_j e^{\i\theta_j} = 0$ imposed by $R_\c$. 

\textbf{Minimal Norm solutions}. One interesting question is that what is the minimal norm representation that achieves the highest objective? For this we can solve the following optimization problem:
\begin{equation}
    \max_{\{m_j,\theta_j\}} \sum_j m_j (e^{3\i\theta_j} + e^{-3\i\theta_j}) - \epsilon \sum_j m^2_j \quad \mathrm{s.t.\ } \sum_j m_j e^{\i\theta_j} = 0
\end{equation}
which achieves the maximal when $m_j = 1 / \epsilon$, $\theta_1 = 2\pi \i /3$ and $\theta_2 = 4\pi \i / 3$ (or vise versa). Note that the optimal $\theta_j$ is fixed no matter how small the regularization coefficient $\epsilon$ is. 

To see that, let $u_j := e^{\i\theta_j}$. Then we have:
\begin{equation}
    \sum_j m_j (u_j + \bar u_j)^3 = \sum_j m_j [u^3_j + 3 u_j \bar u_j (u_j + \bar u_j) + \bar u_j^3] = \sum_j m_j (u^3_j + \bar u^3_j)
\end{equation}
Therefore, letting $x_j := 2\Re{u_j}$, we just need to consider the real part of the objective, and solve the following optimization in $\rr$:
\begin{equation}
    \max_{\{m_j,-2\le x_j\le 2, x_0 = 2\}} \sum_j m_j x^3_j - \epsilon\sum_j m^2_j \quad \mathrm{s.t.\ } \sum_j m_j x_j = 0
\end{equation}
whose solutions give a sufficient condition. Using Lagrangian multiplier, we have:
\begin{equation}
\frac{\partial L}{\partial x_j} = m_j (3x_j^2 - \lambda) = 0,\quad\quad \frac{\partial L}{\partial m_j} = x_j^3 - 2\epsilon m_j - \lambda x_j = 0
\end{equation}
which leads to $\lambda = 3$, $m_j = 1 / \epsilon$ and $x_1 = x_2 = -1$. This corresponds to the solution
\begin{equation}
    z_{pk\cdot} = [1, \omega_3, \bar\omega_3], \quad\quad \mathrm{where\ } p\in \{a,b,c\}
\end{equation}
Note that the original necessary condition is $\va \circ \bar \vb, \vb \circ \bar \vc, \vc \circ \bar \va \in \rr^3$. Considering the possible negativity, the solutions can be written as  
\begin{equation}
    z_{pk\cdot} = [1, \alpha_p \omega_3, \beta_p \bar\omega_3], \quad\quad \mathrm{for\ } p\in \{a,b,c\}
\end{equation}
where $\alpha_p = \pm 1$ and $\beta_p = \pm 1$ with the constraint that $\alpha_a\alpha_b\alpha_c = \beta_a\beta_b\beta_c = 1$.  
\end{proof}
\textbf{Remarks.} Note that this conclusion does not contradict with the constructed solution $\vz_{\syn,\alpha\beta}$ in Eqn.~\ref{eq:synab} in which $\alpha$ and $\beta$ are allowed to be any complex number with magnitude $1$. This is because $\vz_{\syn,\alpha\beta}$ does not satisfy all the constraints in $R_\c$ (but $\vz_{\syn,\alpha\beta} \ringmul \vz_{\nu=\i}$ will) unless $\alpha$ and $\beta$ are real and thus $\pm 1$. 

}

\section{Gradient Dynamics (Proof)}
\begin{figure}
    \centering
    \includegraphics[width=\linewidth]{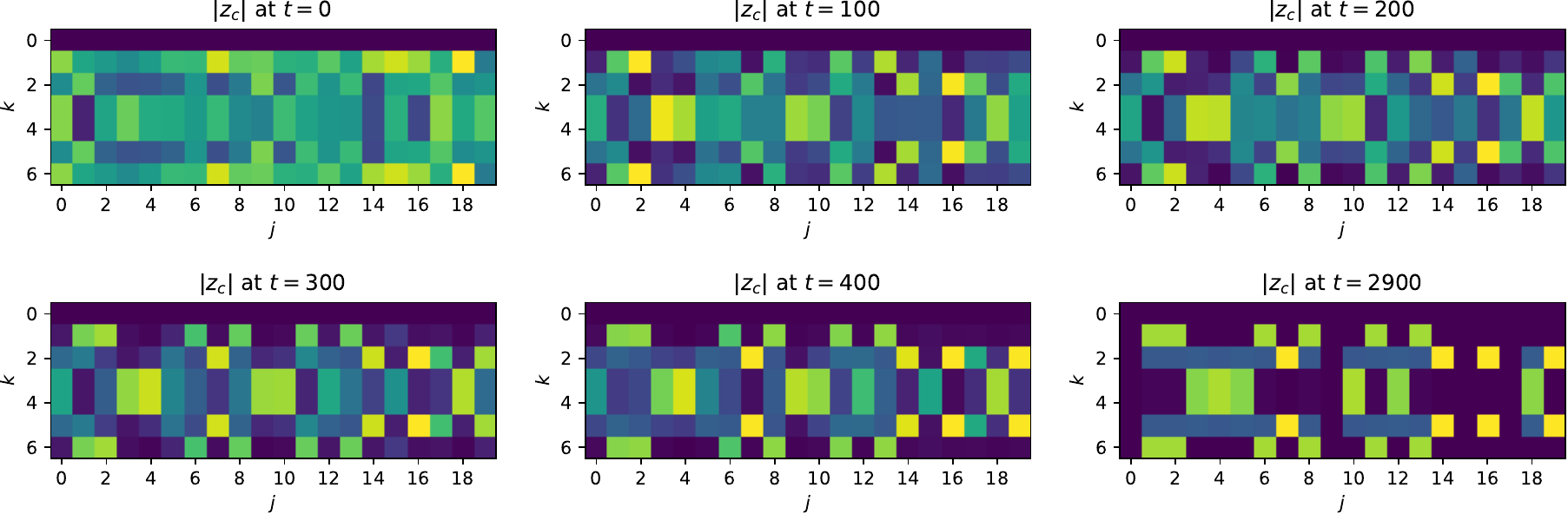}
    \vspace{-0.15in}
    \caption{\small The convergence path of $z_{c\cdot\cdot}$ when training modular addition using Adam optimizer (learning rate $0.05$, weight decay $0.005$). The final solution contains 2 order-6 ($\vz^\bk_{F6}$) and 1 order-4 ($\vz^\bk_{F4}$) solutions. Note that for $z_{c\cdot\cdot}$, unlike Fig.~\ref{fig:solution-structure}, each order-6 solution contains a constant bias term to cancel out the artifacts of order-4 solution (Corollary~\ref{co:foursixsol}). For each hidden node $j$, once a dominant frequency emerges, others fade away.}
    \label{fig:convergence-path}
\end{figure}

Let $\vr = [r_{k_1k_2k}, r_{pk_1k_2k}] \in \cc^{4d^3}$ be a vector of all MPs, and $J := \frac{\partial \vr}{\partial \vz} \frac{\partial \vz}{\partial \cW}$ be the Jacobian matrix of the mapping $\vr = \vr(\vz(\cW))$ in which $\cW$ is the collection of original weights. Note that when we take derivatives with respect to $r$ and apply chain rules, we treat $r$ and its complex conjugate (e.g., $r_{kkk}$ and $r_{-k,-k,-k} = \bar r_{kkk}$) as independent variables. Since we run the gradient descent on $\cW$, will such (indirect) optimization leads to a descent of $\vr$ towards the desired targets (Lemma~\ref{co:globalminimizer})? The following lemma confirms that:
\begin{restatable}[Dynamics of MPs]{lemma}{dynamicsmps}
\label{theorem:dyn-of-mp}
The dynamics of MPs satisfies $\dot \vr = -JJ^* \overline{\nabla_\vr \ell}$, which has positive inner product with the negative gradient direction $-\overline{\nabla_\vr \ell}$. 
\end{restatable}
\begin{proof}
By gradient descent of $\cW$, we have $\dot \cW = - \overline{\nabla_\cW \ell}$. By chain rule, 
we have:
\begin{equation}
    \dot \cW = - \overline{\nabla_\cW \ell} = - \overline{J^\top \nabla_\vr \ell} = - J^* \overline{\nabla_\vr \ell} 
\end{equation}
Then the dynamics of $\vr = \vr(\vz(\cW))$, as driven by the dynamics of $\cW$, is given by 
\begin{equation}
    \dot \vr = J \dot \cW =  -J J^* \overline{\nabla_\vr \ell}
\end{equation}
To show positive inner product, we have:
\begin{equation}
-\overline{\nabla_\vr \ell}^* \dot \vr = \overline{\nabla_\vr \ell}^* JJ^* \overline{\nabla_\vr \ell} = \|J^* \overline{\nabla_\vr \ell}\|_2^2 \ge 0 
\end{equation}
\end{proof}

\loworder*
\begin{proof}
Let $\ord(\vz) = q$ and $\ord(\vz') = q'$. Then $q' | q$. Since both $\vz$ and $\vz'$ are global optimal. Since $r_{kkk}$ is ring homomorphism, we know that $r_{kkk}(\vz) = r_{kkk}(\vz')r_{kkk}(\vy) = 1/2d = r_{kkk}(\vz')$ and thus $r_{kkk}(\vy) = 1$ for all $k\neq 0$.

Let the augmented identity $\ve \in \cZ_q$ be $e_{pmj} = \mathbb{I}(j=0)$. Then $r_{kkk}(\ve) = 1$ for all $k \neq 0$.

We want to construct a path in $\cZ_q$, the space of order-$q$ solutions as follows:
\begin{equation}
    \tilde \vz(t) = \tilde \vy(t) * \vz',\quad\quad 0\le t\le 1
\end{equation}
in which $\tilde \vy(0) = \ve$, $\tilde \vy(1) = \vy$, and $r_{kkk}(\tilde \vy(t)) = 1$ for any $t$. To see why this is possible, pick a continuous family of trajectories $\hat \vy(t;\lambda)$ with $\lambda \in [0, 1]$ so that they satisfies  
\begin{eqnarray}
    \hat \vy(0;\lambda) = \ve,\quad\quad \hat \vy(1;\lambda) = \vy,\quad\quad r_{kkk}(\hat\vy(t;0)) \le 1,\quad\quad r_{kkk}(\hat\vy(t;1)) \ge 1
\end{eqnarray}
which can always be achieved by scaling some trajectory with a factor that depends on $\lambda$. Then by intermediate theorem, there exists $\lambda(t)$ so that $r_{kkk}(\hat\vy(t;\lambda(t))) = 1$ for some $k$. Note that for different frequency $k$ and $k'$, $r_{kkk}$ and $r_{k'k'k'}$ involves disjoint components of $\vz$ so we could find such a path for all $k\neq 0$.

Therefore, for any monomial potential $r$ included in MSE loss (Eqn.~\ref{eq:obj}), we have
\begin{equation}
    r(\tilde \vz(t)) = r(\tilde \vy(t))r(\vz') = \left\{
    \begin{array}{cc}
        \mathrm{finite} \cdot 0 = 0 & r \neq r_{kkk} \\
        1\cdot 1/2d = 1/2d & r = r_{kkk}
    \end{array}
    \right.
\end{equation}
and thus the entire trajectory $\tilde \vz(t) = \tilde\vy(t) * \vz' \in \cZ_q$ connecting $\vz$ and $\ve * \vz'$, which is $\vz'$ in the space of $\cZ_q$, is also globally optimal. 

To see why weight decay regularization leads to lower-order solution, we could simply compare the $\ell_2$ norm of $\vz = \vy * \vz'$ and $\ve * \vz'$. At each frequency $k$, this reduces to the following optimization problem: 
\begin{equation}
    \min \sum_j |a_j|^2 + |b_j|^2 + |c_j|^2,\quad\quad \mathrm{s.t.} \sum_j a_j b_j c_j = 1
\end{equation}
where $a_j := y_{akj}$, $b_j := y_{bkj}$ and $c_j := y_{ckj}$. Since we know that arithmetic mean is no less than geometric mean:
\begin{equation}
    \frac{|a_j|^2 + |b_j|^2 + |c_j|^2}{3} \ge \sqrt[3]{|a_jb_jc_j|^2} 
\end{equation}
We have:
\begin{equation}
    \sum_j |a_j|^2 + |b_j|^2 + |c_j|^2 \ge 3 \sum_j |a_jb_jc_j|^{2/3} \ge 3 
\end{equation}
The last inequality holds because (1) if any $|a_jb_jc_j| \ge 1$, then it holds, (2) if all $|a_jb_jc_j| < 1$, then since $a^x$ is a decreasing function for $a < 1$, $\sum_j |a_jb_jc_j|^{2/3} \ge \sum_j |a_jb_jc_j| \ge |\sum_j a_jb_jc_j| = 1$.

The minimizer is reached when $|a_j| = |b_j| = |c_j|$. Note that if $a_jb_jc_j$ has any complex phase or negative, then in order to satisfy $\sum_j a_jb_jc_j=1$, objective function needs to be larger. So without loss of generality, we could study $a_j=b_j=c_j = x_j\ge 0$ and the optimization problem becomes 
\begin{equation}
    \min \sum_j x_j^2,\quad\quad\mathrm{s.t.} \sum_j x_j^3 = 1, \quad x_j \ge 0
\end{equation}
which has a minimizer at the corners $(1, 0, \ldots)$. This corresponds to $a_j=b_j=c_j = \mathbb{I}(j=0)$, which is the augmented identity $\ve\in\cZ_q$. 
\end{proof}

\infinitelimit*
\begin{proof}
Let $\tilde \ell := \sum_{k} \nabla \tilde \ell_k$. Let's compute the dynamics of MPs following Theorem~\ref{theorem:dyn-of-mp}: $\dot \vr = -JJ^*\overline{\nabla_\vr \tilde\ell}$. 

First it is clear that 
\begin{equation}
    \frac{\partial \tilde \ell}{\partial r_{k_1k_2k}}  = \sum_{k} \frac{\partial \tilde \ell_k}{\partial r_{k_1k_2k}} = -2 \mathbb{I}(k_1=k_2=k) + 2\overline{ r_{k_1k_2k}} 
\end{equation}
So the $(k_1,k_2,k)$ component of $\overline{\nabla_\vr \tilde\ell}$ only contains $r_{k_1k_2k}$.

Then we compute $H := JJ^*$ and show that it is asymptotically diagonal. To see this, each component of $H$, i.e., $h_{k_1k_2k_3,k'_1k'_2k'_3}$ can be computed as the following:
\begin{align}
& h_{k_1k_2k_3,k'_1k'_2k'_3} = \sum_{pmj} \frac{\partial r_{k_1k_2k_3}}{\partial z_{pmj}} \overline{\frac{\partial r_{k'_1k'_2k'_3}}{\partial z_{pmj}}} \\
&= \mathbb{I}(k_1=k_1')\sum_j b_{k_2j}\bar b_{k_2'j} c_{k_3j}\bar c_{k_3'j} \\
&+ \mathbb{I}(k_2=k_2')\sum_j a_{k_1j}\bar a_{k_1'j} c_{k_3j}\bar c_{k_3'j} \\
&+ \mathbb{I}(k_3=k_3')\sum_j a_{k_1j}\bar a_{k_1'j} b_{k_2j}\bar b_{k_2'j}
\end{align}
where $a_{kj} := z_{akj}$, $b_{kj} := z_{bkj}$ and $c_{kj} := z_{ckj}$. Then for component $(k_1k_2k_3,k_1',k_2',k_3')$, if any $k_p \neq k'_p$ for some $p\in \{a,b,c\}$, then the corresponding $z_{pk_pj}\bar z_{pk'_pj}$ has random phase for hidden node $j$, and $h_{k_1k_2k_3,k'_1k'_2k'_3} \rightarrow 0$ when $q\rightarrow +\infty$. 

Combining the two, we know that the dynamics of MPs is decoupled, that is, each $r_{k_1k_2k}$ evolves independently over time.  
\end{proof}

\textbf{Ripple effects}. While Theorem~\ref{lemma:infinitem} only holds at initialization, the resulting decoupled MP dynamics, e.g., $\dd r_{kkk} / \dd t = 1 - r_{kkk}$ that leads to $r_{kkk}(t) = 1 - e^{-t}$, already captures the rough shape of the curve (Fig.~\ref{fig:dynamics-mps} top right). To capture its fine structures (e.g., ripples before stabilization), we can also model the dynamics of the diagonal element in $JJ^*$. Consider a symmetric 1D case on a fixed frequency $k$, where all diagonal $r_{kkk} = r_0 - r$ (where $r_0 = 1/2d$) and all off-diagonal $r_{k_1k_2k} = r$, then  
\begin{equation}
    \dot r = -\dot r_{kkk} = \kappa (r_{kkk} - r_0) = -\kappa r, \quad \dot \kappa = \alpha (r_0 - r_{kkk}) - (1-\alpha) r_{k_1k_2k} - c_0 = (2\alpha - 1) r - c_0
\end{equation}
where $\kappa > 0$ is the diagonal element of $JJ^*$ and $\alpha$ is a coefficient that characterizes the relative strength of two negative gradient $-\overline{\nabla_{r_{kkk}}\ell} = r_0 - r_{kkk}$ and $-\overline{\nabla_{r_{k_1k_2k}}\ell} = - r_{k_1k_2k}$, and $c_0$ is the gradient terms caused by asymmetry and/or other frequencies. This yields a second-order ODE that has complex roots in the characteristic function when $c_0 > 0$.

\section{Extending \ours{} to Group Action Prediction}
\label{sec:extension-to-group-action-pred}
While in this work we mainly focus on Abelian group, \ours{} can be extended to more general \emph{group action prediction}: given a group element $g \in G$ and the current state $x \in \cX$, the goal is to predict $gx \in X$, i.e., the next state after action $g$. Such tasks include modular addition/multiplication in which the group acts on itself (i.e., $\cX = G$), and also includes the transition function in reinforcement learning~\citep{sutton2018reinforcement} and world modeling~\citep{garrido2024learning}, in which an action changes the current state to a new one.   

\textbf{Setup}. Consider a state space $\cX$ and \emph{group action} $G \times \cX \mapsto \cX$ where $g\in G$ is a group element acting on a state $x \in \cX$ to get an update state $gx \in \cX$. It satisfies two axioms (1) the group identity maps everything to itself: $ex = x$, and (2) the group action is compatible with group multiplication: $g(hx) = (gh)x$ for any $g, h\in G$ and $x \in \cX$. 

Equipped with the group action, the state space now can be decoupled into a disjoint of \emph{transitive components}.

\begin{definition}[Transitive group action]
A group action is transitive, if for any $x_1,x_2\in \cX$, there exists $g\in G$ so that $gx_1 = x_2$.
\end{definition}

\begin{figure}
    \centering
    \includegraphics[width=0.8\textwidth]{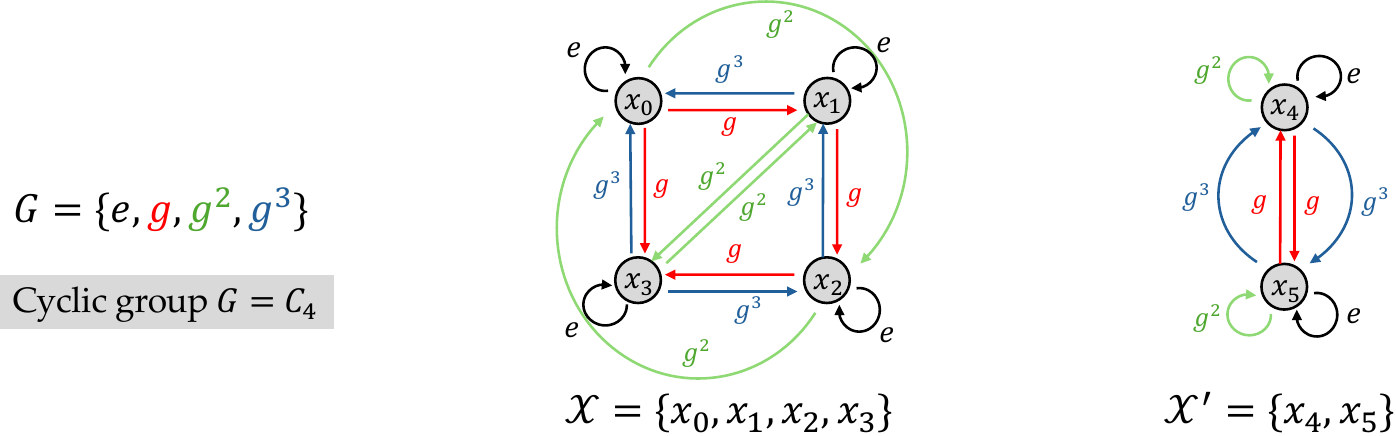}
    \caption{\small An example case of group action on state set $\cX$, $\cX$ can be partitioned into several disjointed components, each is a transitive graph w.r.t the group actions in $G$.}
    \label{fig:graph-decomposition}
\end{figure}

Since the group action is compatible with multiplication, $\cX$ under $G$ will be partitioned into disjoint components $\cX = \bigcup_l \cX_l$ and we can analyze each component separately (Fig.~\ref{fig:graph-decomposition}). 

\textbf{Transitive Group Action}. For each transitive component $\cX$ (dropping $l$ for brevity), under certain conditions, we could define a \emph{state multiplication} operation (a formal definition in Def.~\ref{def:y-mul} in Appendix) so that for any group action $gx\in\cX$, there is an associated state $x'\in\cX$ so that $x' \cdot x = gx$. Furthermore, under the multiplication, $\cX$ itself becomes a group:
\begin{restatable}[$\cX \cong G / G_{x_0}$]{theorem}{quotient}
\label{thm:y-struct}
If the \emph{group stabilizer} $G_{x_0} := \{g | g x_0 = x_0\}$ is a normal subgroup of $G$, then $\cX$ is \emph{isomorphic} to the quotient group $G / G_{x_0}$ and thus forms a group. 
\end{restatable}
Moreover, we can prove that for any group element $g \in G$, there exists $x = \iota_0(g) \in \cX$ so that for any state $x'$, the group action $gx'$ is the same as the state multiplication $x'\cdot x$. Therefore, for group action prediction tasks, we have (note the difference compared to Eqn.~\ref{eq:w-freq-space}):
\begin{equation}
    \vw_j = U_G \left(P_0 \vw^{||}_{j,G} + \vw^\perp_{j,G} \right) + U_{\cX}\vw_{j,\cX}
\end{equation}
where $\vw^{||}_{j,G} \in \rr^{|\cX|}$ is the ``in-graph'' component of $G$, $\vw^{\perp}_{j,G} \in \rr^{|G|}$ is the ``out-of-graph'' component of $G$, and $P_0\in \rr^{|G|\times |\cX|}$ ``lifts'' from $\cX$ to $G$ using $\iota_0$, i.e., $(P_0)_{gx} = 1$ for $g\in\iota^{-1}_0(x)$, and $\vw^{\perp}_{j,G} \perp P_0 \vw^{||}_{j,G}$. Since any $g$ just behaves like $\iota_0(g)$ when acting on $\cX$, our framework can be applied to characterize the learning of $\vw^{||}_{j,G}$. Intuitively, we only learn representation of $G$'s element ``module'' its kernel $G_{x_0}$, since element in the kernel is indistinguishable from each other.

On the other hand, the behavior of $\vw^\perp_{j,G}$ will be influenced by $g$ acting on other graphs, and the final learned representation of a group element $g$ is the direct sum of them. 

\section{Detailed explanation of Sec.~\ref{sec:extension-to-group-action-pred}}
\textbf{Matrix Representation}. Each group element $g$ can be represented by a matrix $R_g$, i.e., its \emph{matrix representation}, so that it respects the group multiplication (i.e., \emph{homomorphism}): $R_{gh} = R_gR_h$ for any group elements $g, h\in G$. 

The dimension of such a representation may differ widely. Some representation can be 1-dimensional (e.g., for Abelian group), while others can be infinitely dimensional. The \emph{permutation representation} $R_g \in \rr^{d\times d}$ maps a one-hot representation $\ve_x \in \rr^{d}$ of an object $\cX$ into its image $\ve_{gx} \in \rr^{d}$, also a one-hot representation. Intuitively, $(R_g)_{jk} = 1$ means that it maps the $k$-th element into the $j$-th element. 
\begin{restatable}[Structure of $R_g$]{lemma}{structrg}
For any $g\in G$, $R_g$ is a permutation matrix. 
\end{restatable}

\begin{restatable}[Summation of $R_g$]{lemma}{summationrg}
\label{lemma:summation_rg}
If the group action is transitive, then $\sum_{g\in G} R_g = \frac{|G|}{d} \vone\vone^\top$. 
\end{restatable}

\subsection{Transitive Case}
To construct the multiplication operation on $\cX$, we first pick reference point $x_0 \in \cX$, and establish a mapping $\iota_0: G \mapsto \cX$: $\iota_0(g) = gx_0$. Note that $\iota_0$ is not necessarily a bijection; in fact we have:
\begin{restatable}[Co-set Mapping $\iota_0$]{lemma}{cosetmapping}
There is a bijection between $\{\iota_0^{-1}(x)\}_{x\in \cX}$ and co-sets $[G : G_{x_0}]$ of \emph{group stabilizer} $G_{x_0} := \{g\in G | g x_0 = x_0\}$, which is a subgroup of $G$ fixing $x_0$.
\end{restatable}

\begin{restatable}[Uniqueness of Multiplication Mapping]{lemma}{uniqymul}
If $G_{x_0}$ is a normal subgroup, then for all $g_1 \in \iota_0^{-1}(x_1)$ and $g_2 \in \iota_0^{-1}(x_2)$, all $g_1g_2 G_{x_0}$ correspond to the same coset.
\end{restatable}

\begin{definition}[The multiplication operator on $\cX$]
\label{def:y-mul}
When $G_{x_0}$ is a normal subgroup, we define \emph{multiplication} on $\cX$: $\cX\times \cX \mapsto \cX$ to be $x_1 x_2 := \iota_0(g_1 g_2 G_{x_0})$ for $x_1 = g_1 x_0$ and $x_2 = g_2 x_0$. Under this definition, $x_0$ is the identity element. 
\end{definition}

\begin{restatable}{lemma}{relyg}
\label{lemma:relyg}
If $g\in \iota_0^{-1}(x)$, then for any $x' \in \cX$, $gx'=xx'$.
\end{restatable}
This means that in terms of group action, the group element $g$ is indistinguishable to $x$ on $\cX$. 

\subsection{General group action}
In this case, $R_g$ can be decomposed into a direct sum of smaller matrices, and all our analysis applies to each of these small matrices. 

In the main text, to simplify the notation, we assume that the group action is transitive, i.e., for any $y, y'\in Y$, there exists $g \in G$ so that $gy = y'$. In the following we will show that for general group actions, the conclusion still follows. 

\emph{Group orbit}. For any $x \in \cX$, Let $G\cdot y := \{gy | g\in G\} \subseteq Y$ be its \emph{orbit}. 
\begin{lemma}
 For $y,y'\in G$, either $G\cdot y = G\cdot y'$ (two orbits collapse) or $G\cdot y \cap G\cdot y' \neq \emptyset$ (two orbits are disjoint). Therefore, orbits form a partition of $\cX$. 
\end{lemma}
Let $X / G := \{G\cdot y | x \in \cX\}$ be the collection of all orbits. The following lemma tells that the matrix representation $R_g$ can be decomposed into a direct sum (i.e., block diagonal matrix) on each orbit.  

\begin{lemma}[Direct sum decomposition of $R_g$]
\label{lemma:direct-sum-group-action}
\begin{equation}
    R_g = \bigoplus_{Y' \in Y / G} R^{Y'}_{g} 
\end{equation}
and each $R^{Y'}_g \in \rr^{|Y'|\times |Y'|}$ is a permutation matrix with $\sum_g R^{Y'}_g = \frac{|G|}{|Y'|} \vone\vone^\top$. 
\end{lemma}
\begin{proof}
By the definition of group orbits, the group action $g$ is closed within each $Y'$. Therefore, $R_g$ is a direct sum (i.e., block-diagonal). 

For each element $x \in \cX'$, let's check its destination under $G$. It is clear that if two group elements $g,h\in G$ maps $\cX$ to the same destination, then 
\begin{equation}
    gy = hy \iff y = g^{-1}hy \iff g^{-1}h \in G_y \iff h = gG_y
\end{equation}
where $G_y$ is the stabilizer of $\cX$, a subgroup of $G$. Therefore, $g$ and $h$ map $\cX$ to the same destination, if and only if they are from the same coset of $G_y$. Therefore, each entry of $\sum_g R^{Y'}_g$ on the column $\cX$ equals to the size of cosets of $G_y$, which is $|G_y|$. Furthermore, for $y_1, y_2 \in Y'$, since they belong to the same orbit, there exists $g$ so that $gy_1 = y_2$ and thus for any $g' \in G_{y_1}$, we have 
\begin{equation}
    g'y_1 = y_1 \iff g g' y_1 = gy_1 = y_2 \iff gg'g^{-1} y_2 = y_2 \iff gg'g^{-1} \in G_{y_2} 
\end{equation}
So there exists bijection between $G_{y_1}$ and $G_{y_2}$. This means that $|G_{y}|$ is constant for any $x \in \cX'$ and thus all elements in $\sum_g R^{Y'}_g$ are equal to $|G| / |Y'|$ (i.e., the number of the group elements that send $\cX$ out to various destinations in $Y'$, divided by the possible distinct destinations $|Y'|$, results in the number of times each destination gets hit).  
\end{proof}

\section{Proofs for the content in Appendix}
\structrg*
\begin{proof}
Since every element needs to have a destination, every column of $R_g$ sums to $1$, i.e., $\vone^\top R_g = \vone^\top$. Then we prove that the mapping $y \mapsto gy$ is a bijection. Suppose there exists $y_1, y_2$ so that $gy_1 = gy_2$. Therefore by compatibility we have: 
\begin{equation}
     g^{-1}(gy_1) = g^{-1}(gy_2) \iff (g^{-1}g)y_1 = (g^{-1}g)y_2 \iff ey_1 = ey_2 \iff y_1 = y_2 
\end{equation}
So any $g$ is a bijective mapping on $\cX$. Since every element of $R_g$ is either $0$ or $1$, $R_g$ is a permutation matrix.  
\end{proof}

\summationrg*
\begin{proof}
Simply apply Lemma~\ref{lemma:direct-sum-group-action} and notice that for transitive group action, $X / G = \{Y\}$.  
\end{proof}

\cosetmapping*
\begin{proof}
First we have 
\begin{equation}
    \iota_0(g) = \iota_0(h) \iff gy_0 = hy_0 \iff y_0 = g^{-1}hy_0 \iff g^{-1}h \in G_{y_0} \iff h \in gG_{y_0}
\end{equation}
So for any $y = g y_0$, all elements in $\iota_0^{-1}(y)$ are also in $gG_{y_0}$ and vice versa. The bijection is:
\begin{equation}
    \iota^{-1}_0(y) \leftrightarrow gG_{y_0}, \quad\quad \mathrm{for\ } y = gy_0
\end{equation}
or equivalently, 
\begin{equation}
    y \leftrightarrow \iota_0(gG_{y_0})
\end{equation}
\end{proof}

\relyg*
\begin{proof}
For $g \in \iota_0^{-1}(x)$, we have $g x_0 = x$. For any $x' = h x_0$, we have:
\begin{equation}
    g x' = g h x_0 = (g h) x_0
\end{equation}
On the other hand, by definition, $xx' := \iota_0(g h G_{x_0}) = (gh) x_0$. So for any $x'$, $gx' = xx'$.
\end{proof}

\section{Additional Experiments}
\label{sec:appendix-additional-exp}
\textbf{Algorithm to extract factorization from gradient descent solutions}. Given the solutions obtained by gradient descent using Adam optimizer, we first compute the corresponding $\vz$ via the Fourier transform (that is, Eqn.~\ref{eq:w-freq-space}). Here $\vz = [z_{pkj}]$ is a $3$-by-$d$-by-$q$ tensor. Here $d = |G|$ and $q$ is the number of hidden nodes in the 2-layer neural networks. 

Then for each frequency $k$, we extract the salient components of $\vz$ by thresholding with a universal threshold (e.g. $0.05$). The number of salient components (e.g., $6$ or $4$) is the order of the per-frequency solution.  

Suppose we now get $\vz^\bk$ for frequency $k$, which is a $3$-by-$6$ (and thus an order-6) solution. Then we enumerate all possible permutation of $6$ hidden nodes ($6!=720$ possibilities) to find one permutation $\tau$ so that $\|z_{pk\tau(\cdot)} - z^\bone_{pk\cdot} \otimes z^\btwo_{pk\cdot}\|$ is minimized, following ring multiplication defined in Def.~\ref{def:operationsinz}. Note that for each permutation, we also need to consider whether $\tilde\vone := [-1,-1,1]$ can be applied to each hidden node $j$ ($\tilde\vone$ is also defined in Tbl.~\ref{tab:poly-construction}). This is because both $\vz_1 \ringadd \vz_2$ and $\vz_1 \ringadd \tilde\vone \ringmul \vz_2$ have exactly the same values on all sum potentials (SPs) we consider, due to the fact that $r(\tilde\vone) = 1$ for any $r \in R_\g \cup R_\c \cup R_\n \cup R_*$. Therefore we call $\tilde\vone$ ``pseudo-1''.

For search efficiency, we therefore first consider the permutation $\tau$ so that $\|z_{ck\tau(\cdot)} - z^\bone_{ck\cdot} \otimes z^\btwo_{ck\cdot}\|$ is minimized, since the component $c$ is invariant to the pseudo-1 transformation $\tilde\vone$, and then for those eligible $\tau$, we search whether $\tilde\vone$ should be applied when considering $p \in \{a,b\}$.

Once we find such $\vz_1$ and $\vz_2$, we convert them into their canonical forms $\tilde\vz_1$ and $\tilde\vz_2$ (Def.~\ref{def:canonical}) to eliminate any possible multiplicative term $\vy$ so that $\vz_1 = \vy \ringmul \tilde\vz_1$. We then compare the canonical forms (up to complex conjugate) with various order-3 and order-2 partial solutions constructed by \ours{}, as detailed in Sec.~\ref{sec:composing-solutions}. If their distance is below a certain threshold (e.g., $< 10\%$ of the norm after normalizing both $\hat\vz_1$ and $\hat\vz_2$), then a match is detected.

\begin{figure}
    \centering
    \includegraphics[width=0.9\textwidth]{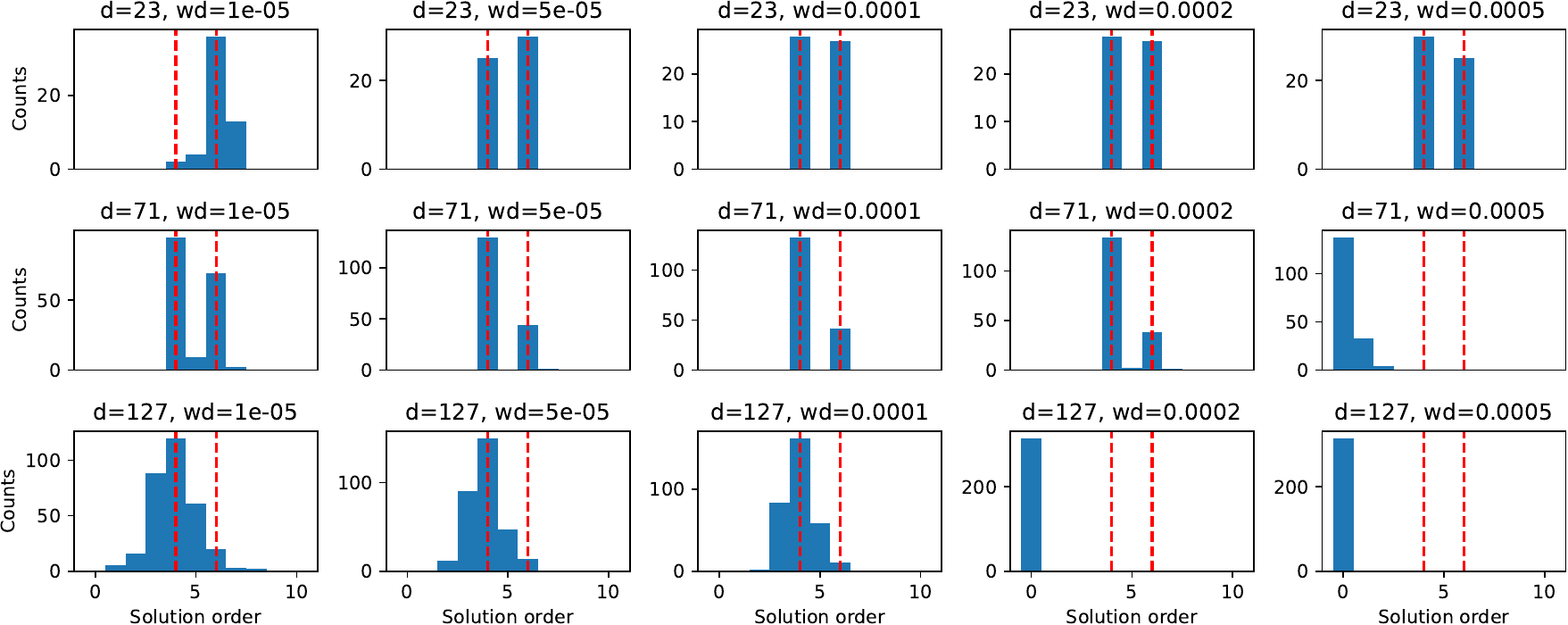}
    \caption{\small Distribution of solutions with hidden size $q = 256$.}
\end{figure}

\begin{figure}
    \centering
    \includegraphics[width=0.9\textwidth]{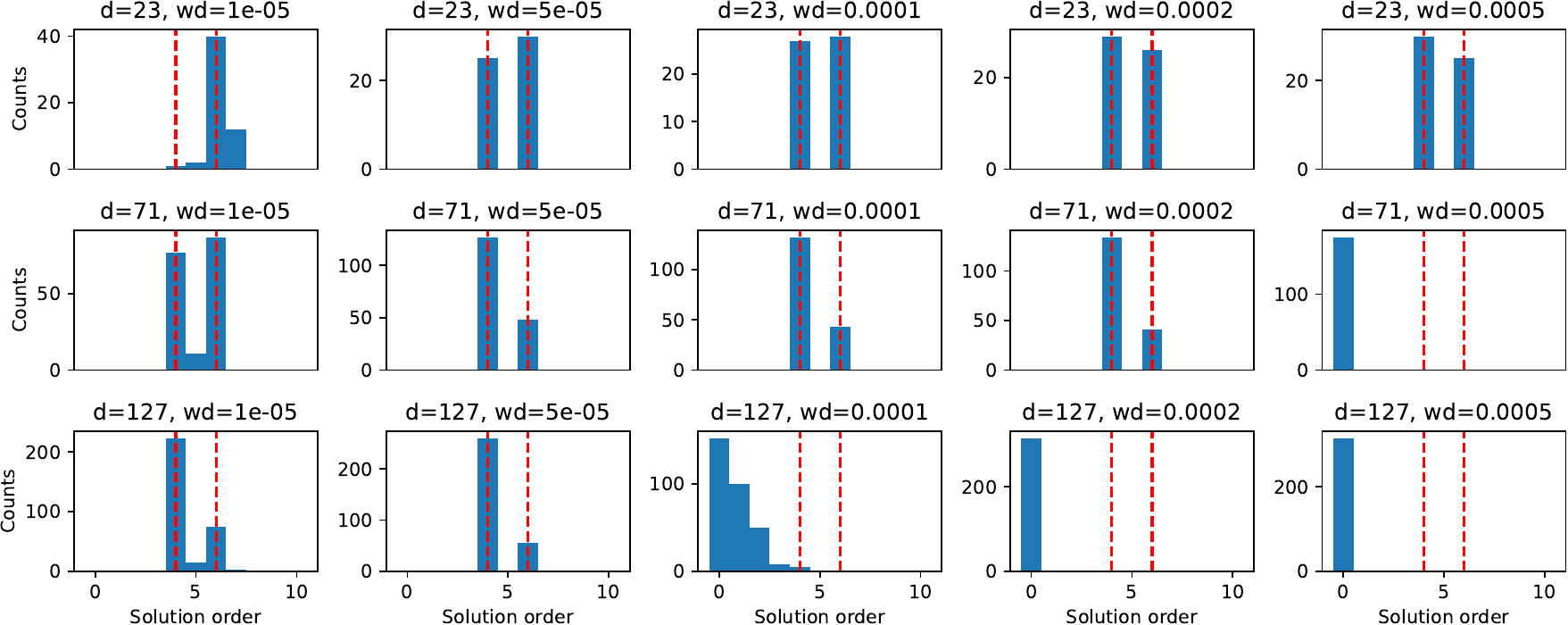}
    \caption{\small Distribution of solutions with hidden size $q = 512$.}
\end{figure}

\begin{figure}
    \centering
    \includegraphics[width=0.9\textwidth]{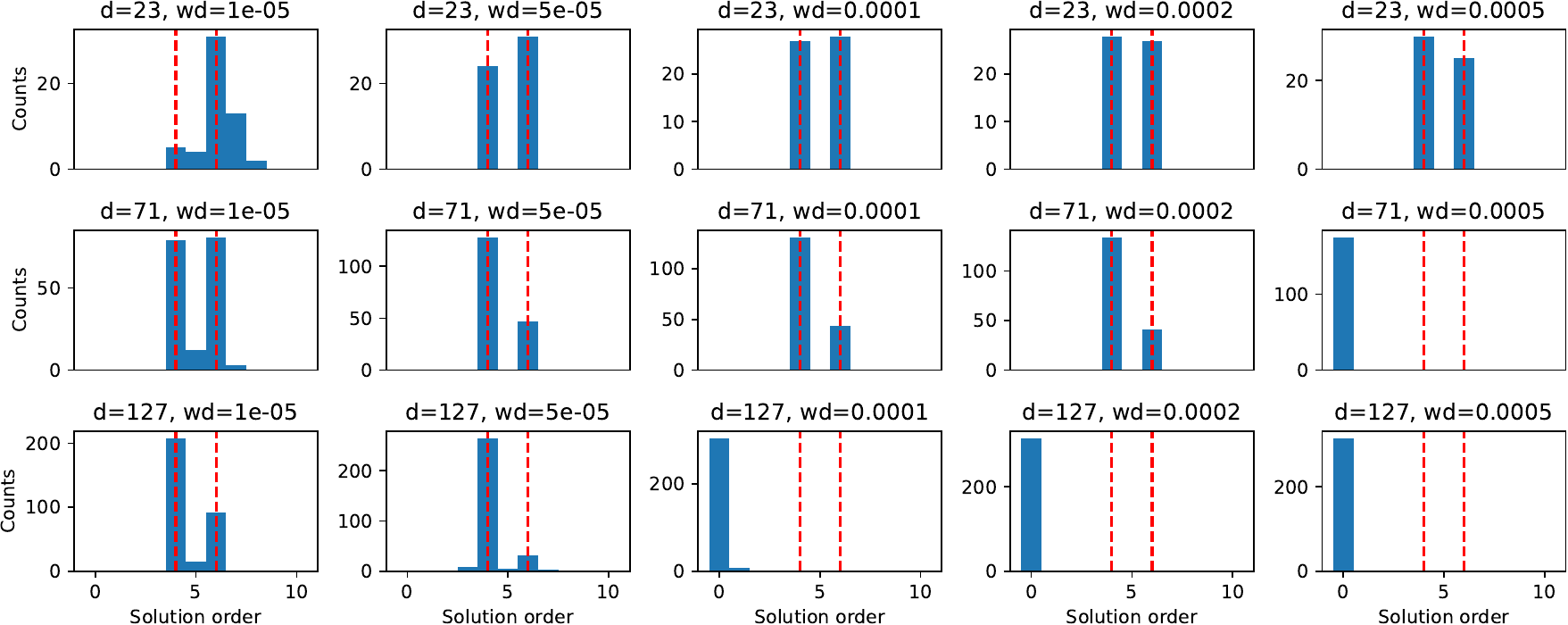}
    \caption{\small Distribution of solutions with hidden size $q = 1024$.}
\end{figure}

\begin{figure}
    \centering
    \includegraphics[width=0.9\textwidth]{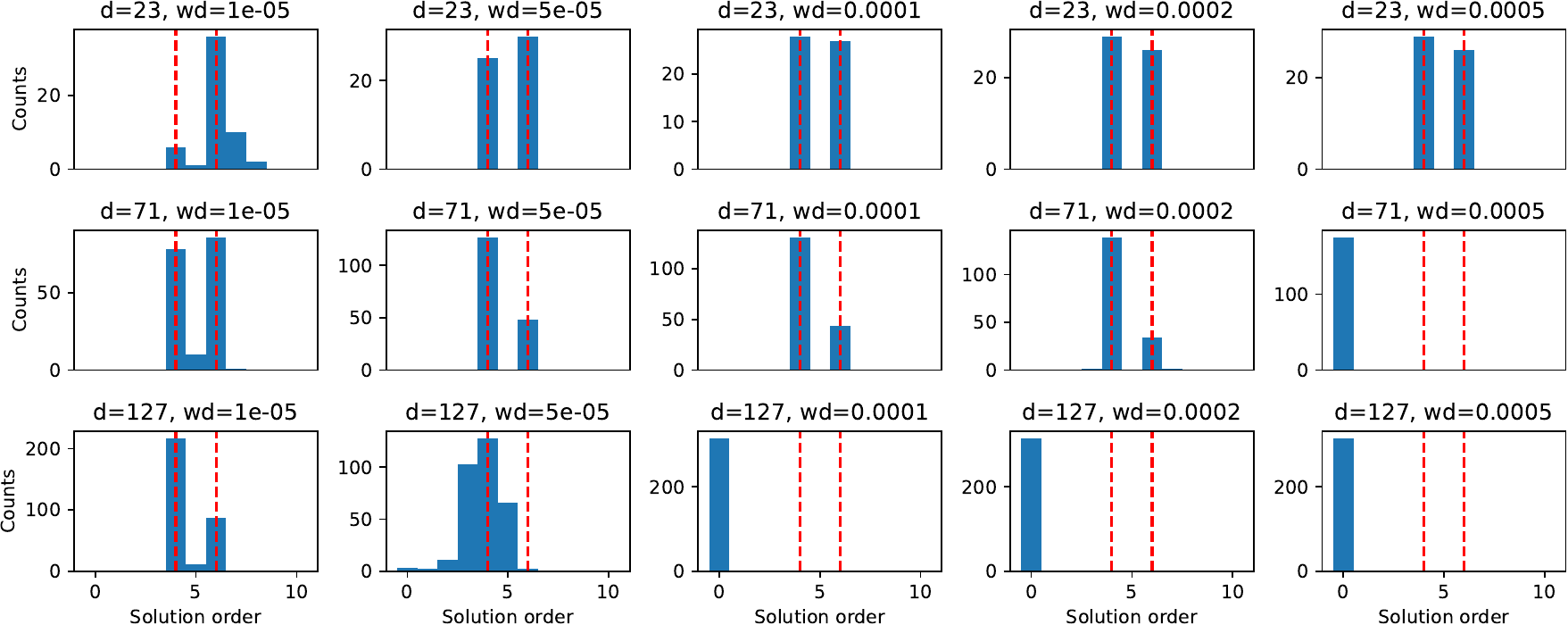}
    \caption{\small Distribution of solutions with hidden size $q = 2048$.}
\end{figure}

\begin{figure}
    \centering
    \includegraphics[width=0.32\linewidth]{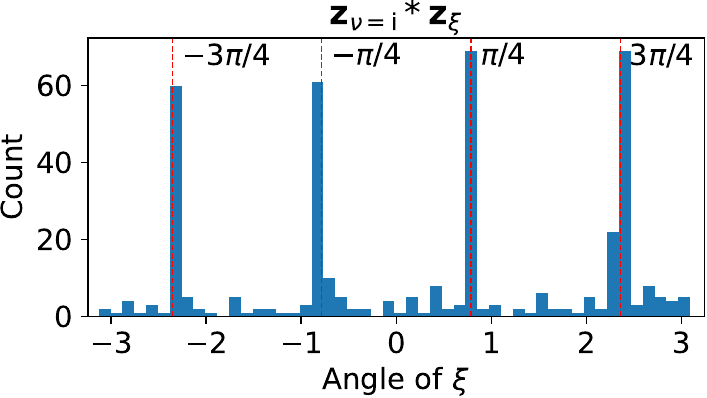}
    \includegraphics[width=0.32\linewidth]{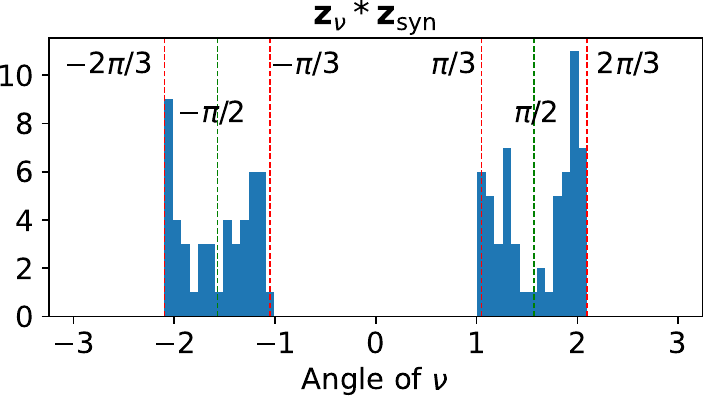}
    \includegraphics[width=0.32\linewidth]{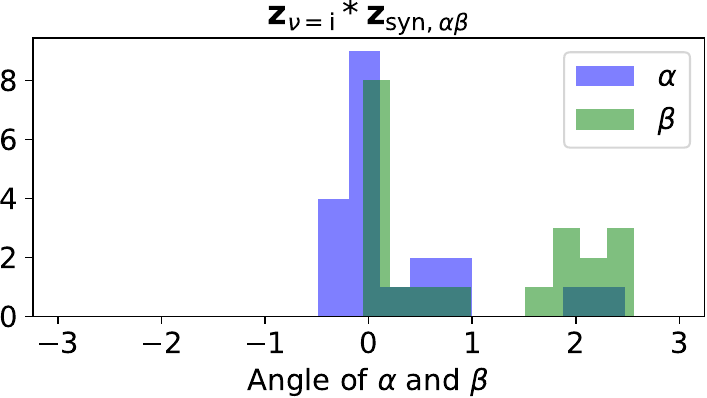}
    \vspace{-0.1in}
    \caption{\small Distribution of free parameters ($\xi$, $\nu$, $\alpha$ and $\beta$, all with magnitude $1$) in three kinds of gradient descent solutions identified by \ours{}. While any value of these parameters makes a global solution, gradient descent dynamics has a particular preference in picking them during optimization.}
    \label{fig:distri-params}
\end{figure}

\end{document}